\documentclass[lettersize,journal]{IEEEtran}
\usepackage{graphicx}
\usepackage{amssymb}
\usepackage{setspace} 

\usepackage{booktabs}
\usepackage{booktabs,makecell,multirow}

\usepackage{booktabs}
\usepackage{hyperref}

\usepackage{widetext}
\usepackage{bm}
\usepackage{color, xcolor}
\usepackage{transparent}
\usepackage{import}
\usepackage[ruled]{algorithm2e}
\usepackage{amsmath}
\usepackage{amsthm}
\allowdisplaybreaks[4]
\usepackage{algorithmic}
\makeatletter

\newcommand{\Rmnum}[1]{\expandafter\@slowromancap\romannumeral #1@}
\makeatother
\usepackage{subfigure} 
\usepackage{threeparttable}

\usepackage{color}
\usepackage{transparent}
\usepackage{graphicx}
\usepackage{import}

\newtheorem{theorem}{Theorem}

\ifCLASSINFOpdf
\else
\fi
 \usepackage{graphicx}
\usepackage{CJKutf8}
\usepackage[T1]{fontenc}
\usepackage{bm}
\usepackage{cite}
\usepackage{amssymb}
\usepackage{amsmath}
 \usepackage{threeparttable}
\begin{document}
	\begin{CJK}{UTF8}{gbsn}
%
\title{Efficient Resource Optimization for Split Federated Learning}
\markboth{IEEE Transactions on  Mobile Computing}%
{Shell \MakeLowercase{\textit{et al.}}: Bare Demo of IEEEtran.cls for IEEE Journals}
%



	
 \author{Wei Wei,~\IEEEmembership{Graduate Student Member,~IEEE}
	and~Xianhao~Chen,~\IEEEmembership{Member,~IEEE}




\thanks{The work was supported in part by the Research Grants Council of Hong Kong under Grant 27213824 and CRS HKU702/24.  
\textit{(Corresponding author: Xianhao Chen)}
}
\thanks{ Wei Wei and Xianhao Chen are with the Department of Electrical and Computer Engineering, the University of Hong Kong, Pok Fu Lam, Hong Kong SAR, 999077. (e-mail: weiwei@eee.hku.hk; 
xchen@eee.hku.hk).
}
}

\maketitle

\begin{abstract}
Split federated learning (SFL) has emerged as a powerful paradigm for model training at the edge. However, SFL inherently involves discrete decision variables for model splitting and resource allocation, resulting in a challenging mixed-integer problem. Consequently, prior optimization schemes for SFL are either \textit{heuristic} or \textit{computationally inefficient}, which cannot handle large-scale user populations. To address this limitation, this work establishes an efficient optimization framework for SFL under resource-constrained networks. Our framework jointly optimizes model splitting and resource allocation to minimize training cost, which is defined as the weighted sum of latency and energy costs. We first study the model splitting problem and develop a polynomial-time algorithm that achieves the global optimum. Then, we extend the approach to the joint model splitting and resource allocation problem. In this case, we formulate it as a two-dimensional master problem and develop an efficient approximation method with a $(1+\epsilon)$-approximation guarantee. Extensive experiments show that the proposed approach provides efficient solutions to strike the optimal energy--latency tradeoff.
\end{abstract}


%
\IEEEpeerreviewmaketitle

\section{Introduction}
Edge learning enables artificial intelligence (AI) computation close to the data source, enhancing privacy, saving bandwidth, and enabling personalization of edge devices~\cite{  10040976,10843329,11482856,ZW_spectrum,10835069}. However, the rapid expansion of AI models presents unprecedented computational challenges. State-of-the-art learning is inherently incompatible with the limited hardware resources of edge devices. For instance, even powerful edge devices such as the NVIDIA Jetson AGX Xavier struggle to train models with more than 50 million parameters~\cite{jetson-agx}. Given these challenges, split federated learning (SFL)~\cite{vepakomma2018split, lin2023split,10934144}, which utilizes a client-server architecture, presents a promising solution. In SFL, the AI model is divided into two segments: the initial layers are trained on clients over local data, whereas the subsequent layers are trained on a server that provides more powerful computing resources. Moreover, a federated server (or simply Fed server) periodically aggregates client-side sub-models, akin to the federated learning (FL) paradigm~\cite{9084352,konevcny2016federated, 10233897}, to synchronize models across multiple clients in parallel. SFL alleviates the hardware constraints of clients by offloading the most demanding computations to more powerful servers, thereby recognized as a promising framework to enable large-scale model training on edge devices.


\begin{figure}[t!]

         \subfigure[\centering Illustration of SFL framework.] {
\centering\includegraphics[width=4.1cm]{ 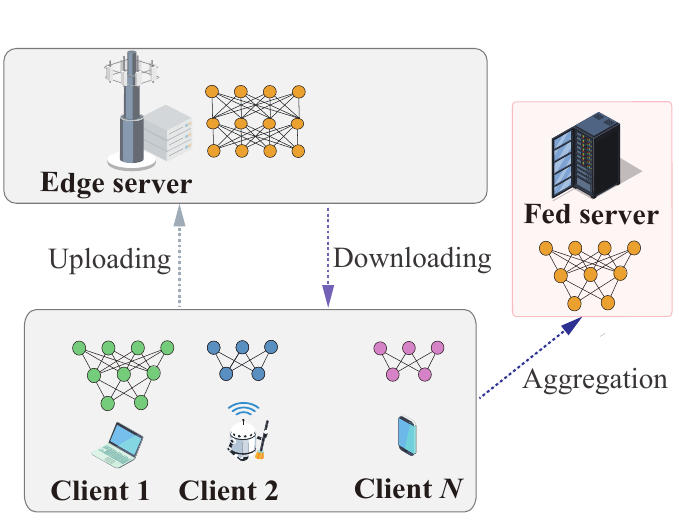}
	}
         \subfigure[\centering Illustration of client-side idle time.] {
\centering\includegraphics[width=4.1cm]{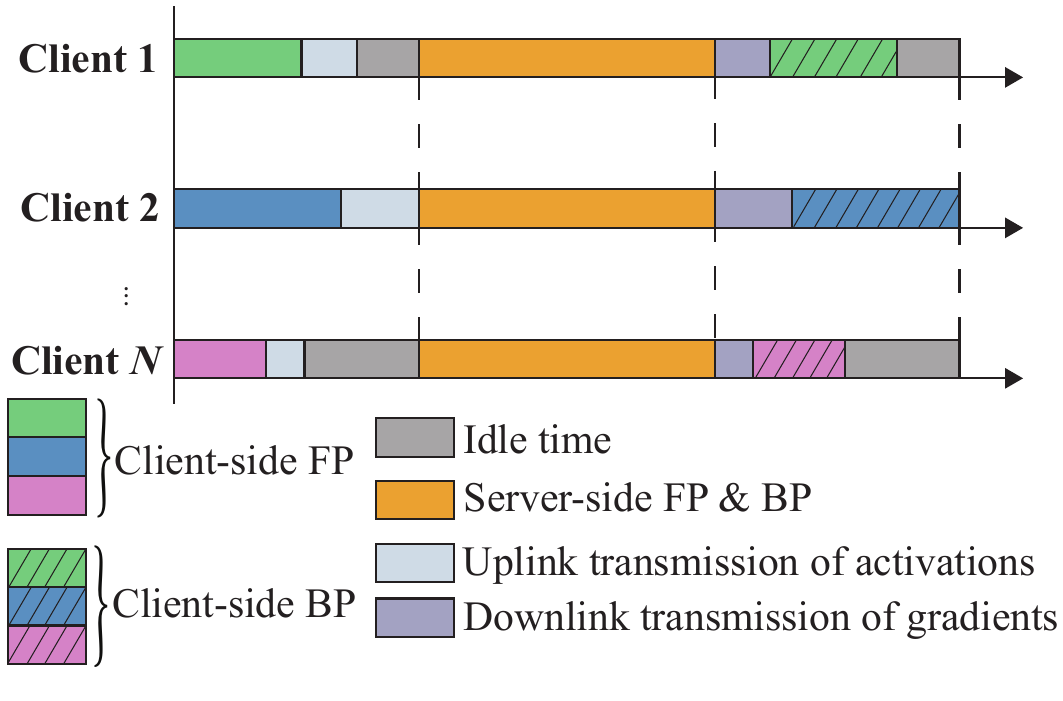
        }

}

	\centering\caption{\textcolor{black}{Illustration of SFL workflow. (a)  SFL scenario includes uploading, downloading and aggregation processes. (b) Since the server starts only after all clients finish local FP and activation uploading, and the next round begins only after all clients complete BP and gradient downloading, moderately reducing the GPU frequency and transmission power of faster clients does not affect the overall training time but can save energy. } }
		\label{sfig:motivation}
\end{figure}

As with other edge learning paradigms, resource optimization is essential for SFL, since its performance critically depends on how computation, communication, and model-splitting decisions are orchestrated under limited network and computing resources~\cite{11049940,10980018,11447454,11458847}. 
On the client side, heterogeneous communication and computing capabilities, combined with model-splitting choices, often create a straggler bottleneck where the slowest clients dominate the overall training time (see Fig.~\ref{sfig:motivation}(b)). Mitigating this requires strategic resource management. First, \textit{optimizing power and computing resource allocation} for heterogeneous clients can save energy while prioritizing resources for slower clients to accelerate training. Second, the \textit{model splitting decision} determines how computation and communication loads are distributed between the client and server, thereby affecting both energy consumption and end-to-end latency, as shown in Fig.~\ref{sfig:motivation1}. Thus, realizing efficient SFL at the edge requires the joint optimization of computation, communication, and model-splitting decisions.

In recent years, the SFL literature has addressed resource allocation for improving training efficiency, covering client/server workload coordination~\cite{10522472,9839238,10298247,9944188}, model splitting~\cite{10706803,11049940,10980018,meng2026asfladaptivemodelsplitting}, wireless resource allocation~\cite{10526451,meng2026asfladaptivemodelsplitting,11219052,10193767,chen2026siftmoesimilarityawareenergyefficientexpert}, and aggregation control~\cite{11049940,10980018}. While these advances have significantly improved SFL performance, two fundamental challenges in SFL remain unresolved: First, although the model-splitting decision is a core component of SFL, existing studies still lack a scalable polynomial-time solution. It remains an open question whether cut-layer optimization in SFL, due to its combinatorial structure, can admit an exact polynomial-time algorithm. Resolving this problem is crucial for enabling scalable SFL optimization with large user populations (e.g., on the order of hundreds of users or more). Second, beyond model partitioning, the performance of SFL is strongly influenced by GPU operating frequencies and transmission powers, which jointly determine the energy–latency trade-off in heterogeneous edge environments. Nevertheless, the joint optimization of model partitioning and resource allocation substantially increases the complexity of the problem, rendering the overall optimization task more challenging.

Motivated by these challenges, this paper aims to answer two research questions: 
\begin{itemize}
    \item[]\textbf{Q1:} Can model-splitting optimization in SFL be solved optimally in polynomial time?  
    \item[]\textbf{Q2:} Can model-splitting decisions, GPU frequency scaling, and transmission power control be jointly optimized within a unified framework? 
\end{itemize}

This paper firmly answers both questions through the following contributions.

\begin{figure}[t!]

\centering
\includegraphics[width=8cm]{ 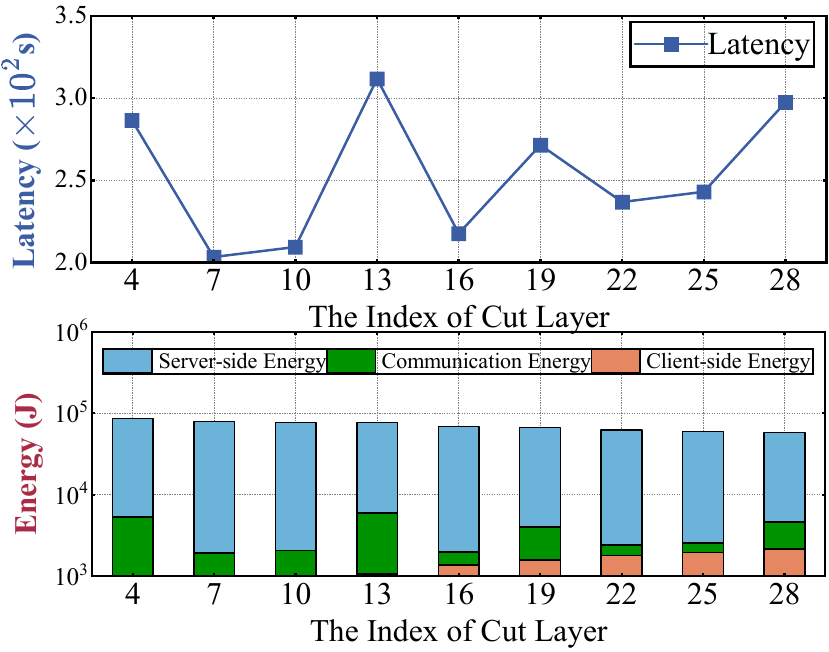}
	\flushleft

        
	\centering\caption{\textcolor{black}{Impact of cut layer on energy and latency.  The experiments are conducted
with ResNet-50 on the CIFAR-10 dataset under non-IID settings. Other configurations are identical to experimental settings in Section \ref{simulations}.
} }
		\label{sfig:motivation1}
\end{figure}

\begin{itemize}
    \item We establish a unified resource optimization framework for SFL by jointly considering model splitting, GPU frequency scaling, and transmission power control.
    
    \item For the model-splitting problem, we derive an optimal polynomial-time solution approach.
    
    \item For the joint model-splitting and resource optimization problem, we formulate it as a two-dimensional master problem and develop a polynomial-time approximation scheme with  a ($1+\epsilon$) - approximation guarantee.
    
    \item Extensive experiments demonstrate that the proposed method consistently outperforms baseline schemes in terms of both latency--energy tradeoff and running-time efficiency.
\end{itemize}

The remainder of this paper is organized as follows. Section \ref{sec:related} reviews related work. Section \ref{sec:system-model} introduces the system model and the proposed resource optimization framework for SFL. Section \ref{sec:problem_formulation} presents the problem formulation. Section \ref{subsec:cutonly} considers a special case with model-splitting optimization and derives a polynomial-time optimal solution.  Section~\ref{solu_appro} investigates the general joint optimization problem, reformulates it as a two-dimensional master problem, and develops an approximation algorithm. Section \ref{simulations}  presents the numerical results. Finally, Section \ref{conclusion} concludes this paper.

\section{Related Work}
\label{sec:related}
\textbf{Resource management in distributed learning and edge AI.}
Recent studies have shown that runtime resource control, especially GPU frequency scaling and power management, can substantially affect the time and performance of deep learning. For example, Drechsler et al.~\cite{drechsler2023energy} show that lowering GPU frequency can significantly reduce power consumption for CNN inference with only marginal impact on computation time. More recent studies extend this observation to training tasks. Maliakel et al.~\cite{maliakel2026characterizingllminferenceenergyperformance} systematically analyze the impact of dynamic voltage and frequency scaling across different models and tasks, while throttLL’eM~\cite{10946751} exploits GPU frequency scaling to improve training efficiency under latency constraints. 
Beyond frequency control, several works have explored other forms of runtime resource control for improving the time--resource tradeoff of deep learning systems. Zeus~\cite{You2023119} jointly controls batch size and GPU power limits for recurring DNN training jobs to better balance training time and energy consumption. EnvPipe~\cite{Choi2023851} improves training efficiency by leveraging idle pipeline intervals and dynamically adjusting streaming multiprocessor frequencies without sacrificing performance. PCCL~\cite{10818209} focuses on distributed training and identifies more resource-efficient GPU frequency settings for collective communication calls, thereby reducing communication overhead with negligible throughput degradation. Perseus~\cite{10.1145/3694715.3695970} further improves the time--resource tradeoff of large-model training by optimizing computation scheduling along the performance frontier.

These studies demonstrate that resource control is an effective mechanism for improving the efficiency of deep learning workloads in both distributed learning and edge AI systems. However, they are primarily designed for conventional training, inference, or distributed GPU platforms. Thus, they either rely on fixed model partitioning or fail to address partitioning across heterogeneous edge devices.

\textbf{Resource management in SFL.}
Recent studies on SFL have increasingly focused on resource allocation for improving training efficiency in heterogeneous edge environments. For example, ESFL~\cite{10522472} optimizes client-side workload distribution and server-side resource allocation over heterogeneous wireless devices, while Wu \textit{et al.}~\cite{10526451} and HMSFL~\cite{10706803} further incorporate model splitting, communication/computing resource allocation, and hierarchical split designs to improve training efficiency. In addition, AdaptSFL and HSFL~\cite{11049940,10980018} show that split decisions and aggregation intervals can be adjusted to better balance communication and computation during training, thereby affecting communication overhead, local computation burden, and convergence speed. More recently, ASFL~\cite{meng2026asfladaptivemodelsplitting} combines adaptive model splitting with wireless resource allocation, and Wang \textit{et al.}~\cite{11219052} further consider resource-constrained edge environments with unreliable wireless links. Nevertheless, most existing approaches focus on latency reduction, without considering a unified latency-energy tradeoff. More importantly, the joint optimization of model splitting and resource allocation remains largely open, where polynomial-time algorithms with provable approximation guarantees are still lacking.

\section{System Model}
\label{sec:system-model}
\begin{table}[t]\label{notation}
\caption{{Summary of Important Notations.}}
{\small 
  \renewcommand{\arraystretch}{1.2}{
  \setlength{\tabcolsep}{0.1mm}{
\begin{tabular}{ll}
			\toprule
			\toprule
\textbf{Notation}                                                              & ~~~\textbf{Description}  \\ \hline
~~~$\mathcal{N}$ &  ~~~The set of all clients     \\
~~~$N$ &  ~~~The number of clients     \\
~~~$L$           &  ~~~The number of layers in the global model  \\
~~~$L_{c,i}$           &  ~~~The cut layer of client $i$   \\~~~$n_i$ & ~~~The number of  GPU FLOPs per cycle of client $i$
\\~~~$n_{\mathrm{S}}$ & ~~~The number of  GPU FLOPs per cycle of the edge server
\\~~~$f_{i,j}$ & ~~~The clock frequency of client $i$’s GPU
\\~~~$f_{\mathrm{S},j}$ & ~~~The clock frequency of the edge server’s GPU\\
$a_i(L_{c,i})$& ~~~The activations of client $i$ at the $L_{c,i}$-th cut layer \\
~~~$\rho_l$& ~~~The FP computational cost  of the $l$-th DNN layer \\
~~~$\varpi_l$& ~~~The BP computational cost  of the $l$-th DNN layer \\
~~~$B_i^{\mathrm{U}}$& ~~~The uplink bandwidth allocated to client $i$ \\
~~~$B^{\mathrm{D}}_i$ & ~~~The downlink bandwidth allocated to client $i$ \\
~~~$P^{\mathrm{TX}}_{i,j}$& ~~~The transmit power of client $i$ at the $j$-th stage \\
~~~$P^{\mathrm{TX}}_{\mathrm{S},j}$& ~~~The transmit power of the edge server at the $j$-th stage \\
~~~$h_i$& ~~~The uplink channel fading coefficient  \\
~~~$h_{\mathrm{S}}$& ~~~The downlink channel fading coefficient  \\
~~~$N_0$& ~~~The noise power spectral
density  \\
$g_i(L_{c,i})$& ~~~The gradients of client $i$ at the $L_{c,i}$-th cut layer  \\
~~~$G_i$& ~~~The coefficient  depends on the chip  of client $i$ 
\\
~~~$G_{\mathrm{S}}$& ~~~The coefficient depends on the chip  of the edge server
\\
~~~$u_{i,j}$& ~~~The control input of client $i$ at the $j$-th stage
\\
~~~$\lambda$& ~~~Balances energy efficiency and training latency
\\ \bottomrule
\end{tabular}}}
}
\end{table}
\subsection{SFL Architecture}
We consider an SFL system with a set of clients $\mathcal{N}=\{1,\dots,N\}$, an edge server that executes the server-side sub-model, and a federated server that periodically aggregates model parameters. 
For each client $i$, the global DNN model is partitioned into a \textit{client-side sub-model}, 
which comprises the initial DNN layers up to the cut layer indexed by $L_{c,i}$, 
and a \textit{server-side sub-model}, which includes the remaining DNN layers from $L_{c,i}+1$ to the final output layer. 
Client $i$ performs local forward pass (FP) up to the cut layer and transmits the corresponding intermediate activations to the edge server, which continues the remaining computation and returns the cut-layer gradients to the client for completing backward pass (BP).

\subsection{Stages of SFL}
Within one SFL round, each client $i$ performs a five-stage pipeline\footnote{After every $I$ local training rounds, the federated server performs model aggregation to update the global parameters across all clients.
The aggregation process is executed on a much slower timescale than the intra-round SFL pipeline.
Thus, aggregation time can be appended as a constant offset at the end of $I$ local training rounds and the resource optimization of model aggregation is omitted in this paper.} that captures both computation and communication operations. 
Specifically, the stages include:  
1) client-side FP,
2) uplink transmissions of activations,
3) server-side FP and BP,
4) downlink transmissions of gradients, and
5) client-side BP. We denote these stages as $\mathcal{J}_i^{\mathrm{FP}}$, $\mathcal{J}_i^{\mathrm{U}}$, $\mathcal{J}_i^{\mathrm{FP+BP}}$, $\mathcal{J}_i^{\mathrm{D}}$, and $\mathcal{J}_i^{\mathrm{BP}}$, respectively.

\begin{enumerate}

\item \textbf{Client-side FP.}
Client $i$ executes the first $L_{c,i}$ DNN layers on its local data,
producing activations $a_i(L_{c,i})$ at the cut layer.
The FP time of client $i$ is modeled as
\begin{equation}
s_{i,j}(f_{i,j},L_{c,i})
= \sum_{l=1}^{L_{c,i}} \frac{\rho_l}{f_{i,j} n_i}, \quad j=1,
\end{equation}
where $\rho_l$ denotes the FP computational cost (in FLOPs) of the $l$-th DNN layer, $n_i$ denotes the number of  GPU
FLOPs per cycle and $f_{i,j}$ represents the clock frequency of client $i$'s GPU.

\item \textbf{Uplink transmissions of activations.}
Client $i$ transmits the activations $a_i(L_{c,i})$ to the edge server. 
The data size of $a_i(L_{c,i})$ is denoted by $d^{\mathrm{U}}_i(L_{c,i}) = \text{bits}(a_i(L_{c,i}))$, and the corresponding transmission time is given by
\begin{equation}
s_{i,j}(P_{i,j}^{\mathrm{TX}},L_{c,i})
= \frac{d_i^{\mathrm{U}}(L_{c,i})}
{B_i^{\mathrm{U}}\log\Big(1+\frac{P_{i,j}^{\mathrm{TX}}|h_i|^2}{B_i^{\mathrm{U}}N_0}\Big)},
\quad j=2.
\end{equation}
Here, 
$B_i^{\mathrm{U}}$ denotes the uplink bandwidth allocated to client $i$, 
$P^{\mathrm{TX}}_{i,j}$ is the transmit power of client $i$ at the $j$-th stage, 
$h_i$ denotes the uplink channel fading coefficient, $h_i^2$ is the corresponding channel gain, 
and $N_0$ is the noise power spectral density.

\item \textbf{Server-side FP and BP.}
Upon receiving $a_i(L_{c,i})$, the edge server performs the remaining forward and backward passes from layer $L_{c,i}+1$ to $L$, and outputs the corresponding gradient $g_i(L_{c,i})$. 
The data size of $g_i(L_{c,i})$ is denoted by $d^{\mathrm{D}}_i(L_{c,i}) = \text{bits}(g_i(L_{c,i}))$. 
The server-side computation time is
\begin{equation}
s_{i,j}(f_{\mathrm{S},j},L_{c,i})
= \sum_{l=L_{c,i}+1}^{L}\frac{\rho_l+\varpi_l}{f_{\mathrm{S},j}n_{\mathrm{S}}},
\quad j=3,
\end{equation}
where 
$\varpi_l$ denotes the BP computational cost (in FLOPs) of the $l$-th DNN layer, 
$f_{\mathrm{S},j}$ is the  clock frequency  of the edge server's GPU, 
and $n_{\mathrm{S}}$ denotes the number of GPU
FLOPs per cycle of edge server.

\item \textbf{Downlink transmissions of gradients.}
The edge server transmits $g_i(L_{c,i})$ back to client $i$ through the downlink channel. 
The data size of $g_i(L_{c,i})$ is denoted by $d^{\mathrm{D}}_i(L_{c,i})$, and the corresponding transmission time is modeled as
\begin{equation}
s_{i,j}(P_{\mathrm{S},j}^{\mathrm{TX}},L_{c,i})
= \frac{d^{\mathrm{D}}_i(L_{c,i})}
{B^{\mathrm{D}}_i\log\Big(1+\frac{P_{\mathrm{S},j}^{\mathrm{TX}}|h_{\mathrm{S}}|^2}{B^{\mathrm{D}}_iN_0}\Big)},
\quad j=4,
\end{equation}
Here, 
$B^{\mathrm{D}}_i$ denotes the downlink bandwidth, 
$P^{\mathrm{TX}}_{\mathrm{S},j}$ is the transmit power of the edge server at the $j$-th stage, and
$h_{\mathrm{S}}$ represents the channel fading coefficient.

\item \textbf{Client-side BP.}
After receiving the gradients from the edge server, 
client $i$ performs BP on its local DNN layers, 
and subsequently updates the parameters of its client-side sub-model. 
The service time for this stage is modeled as
\begin{equation}
s_{i,j}(f_{i,j},L_{c,i})
= \sum_{l=1}^{L_{c,i}}\frac{\varpi_l}{f_{i,j} n_i},
\quad j=5.
\end{equation}

\end{enumerate}

\subsection{Per-round Latency}
Based on the defined five-stage SFL, we consider a \emph{parallel execution model} across clients within each SFL round. 
Under this model,  the  per-round training latency of client-side FP and uplink  is dominated by the slowest client.  Moreover, the edge server maintains a shared server-side sub-model and must process the server-side FP and BP for all clients sequentially. Thus, server-side workloads contribute additively, yielding an aggregate latency term. Similarly,  the per-round training latency of client-side BP and downlink  is dominated by the slowest client.
Accordingly, we define the per-round latency as
\begin{align}
&T_{\mathrm{r}}(f_{i,j}, P_{i,j}^{\mathrm{TX}},f_{\mathrm{S},j},P_{\mathrm{S},j}^{\mathrm{TX}},L_{c,i})=\notag\\
&
\max_{i\in\mathcal{N}}\Big\{s_{i,1}(f_{i,1},L_{c,i})+s_{i,2}(P_{i,2}^{\mathrm{TX}},L_{c,i})\Big\}
+\sum_{i\in\mathcal{N}} s_{i,3}(f_{\mathrm{S},3},L_{c,i})\notag\\
&+\max_{i\in\mathcal{N}}\Big\{s_{i,4}(P_{\mathrm{S},4}^{\mathrm{TX}},L_{c,i})+s_{i,5}(f_{i,5},L_{c,i})\Big\}.
\label{eq:round_latency_parallel}
\end{align}

\subsection{Energy Model}
We next quantify the per-round energy consumption by incorporating both computing and communication energy in the five-stage SFL pipeline.
For each client $i\in\mathcal N$, we define the energy consumed at the $j$-th stage ($j\in\{1,\dots,5\}$) under cut layer $L_{c,i}$ as
\begin{align}
&E_{i,j}(f_{i,j}, P_{i,j}^{\mathrm{TX}},f_{\mathrm{S},j},P_{\mathrm{S},j}^{\mathrm{TX}},L_{c,i})\triangleq \notag\\
&\left\{ 
\begin{array}{ll}
\dfrac{G_i}{{n_i}^3}
\dfrac{\big(\!\sum\limits_{l=1}^{L_{c,i}}\rho_l\big)^3}{
\big[s_{i,j}(f_{i,j},L_{c,i})\big]^2}, 
\ j=1,\\
\dfrac{B^{\mathrm{U}}s_{i,j}(P_{i,j}^{\mathrm{TX}},L_{c,i})N_0}{h_i^2}
\!\left(2^{\frac{d_i^{\mathrm{U}}(L_{c,i})}{
B^{\mathrm{U}}s_{i,j}(P_{i,j}^{\mathrm{TX}},L_{c,i})}}-1\right),\ 
 j=2,\\
\dfrac{G_{\mathrm{S}}}{{n_{\mathrm{S}}}^3}
\dfrac{\left(\sum\limits_{l=L_{c,i}+1}^{L}(\rho_l+\varpi_l)\right)^3}{
\big[s_{i,j}(f_{\mathrm{S},j},L_{c,i})\big]^2},\
 j=3\\
\dfrac{B^{\mathrm{D}}s_{i,j}(P_{\mathrm{S},j}^{\mathrm{TX}},L_{c,i})N_0}{h_{\mathrm{S}}^2}\cdot\left(2^{\frac{d_i^{\mathrm{D}}(L_{c,i})}{
B^{\mathrm{D}}s_{i,j}(P_{\mathrm{S},j}^{\mathrm{TX}},L_{c,i})}}-1\right),\
 j=4,\\
 \dfrac{G_i}{{n_i}^3}
\dfrac{\big(\!\sum\limits_{l=1}^{L_{c,i}}\varpi_l\big)^3}{
\big[s_{i,j}(f_{i,j},L_{c,i})\big]^2},\
 j=5,
\end{array}
\right.
\label{eq:theta_energy1}
\end{align}
where  the coefficients $G_i$ and $G_{\mathrm{S}}$ [in $\text{Watt/(cycle/s)}^3$] depend on the chip architecture of client~$i$ and the edge server~\cite{9461628}.
Accordingly, the total energy consumption of one SFL round is given by
\begin{align}
&E_{\mathrm{r}}(f_{i,j}, P_{i,j}^{\mathrm{TX}},f_{\mathrm{S},j},P_{\mathrm{S},j}^{\mathrm{TX}},L_{c,i})
=\notag\\
&\qquad \qquad  \qquad \qquad  \sum_{i\in\mathcal{N}}\sum_{j=1}^5E_{i,j}(f_{i,j}, P_{i,j}^{\mathrm{TX}},f_{\mathrm{S},j},P_{\mathrm{S},j}^{\mathrm{TX}},L_{c,i}).
\label{eq:round_energy}
\end{align}
\subsection{Control Input of Each SFL Stage}
\label{subsec:ct-dynamics} 
The control input $u_{i,j}(t)$ represents the adjustable resource variables that 
govern the evolution rate $f_{i,j}\!\big(z_{i,j}(t),u_{i,j}(t)\big)$ and thus determine 
the stage service time $s_{i,j}\big(u_{i,j}(t),L_{c,i}\big)$.
Formally,
\begin{equation}
u_{i,j}(t)=
\begin{cases}
f_{i,j}(t), & j\in\{1,5\},\\[3pt]
f_{\mathrm{S},j}(t), & j =3,\\[3pt]
P_{i,j}^{\mathrm{TX}}(t), & j =2,\\[3pt]
P_{\mathrm{S},j}^{\mathrm{TX}}(t), & j =4.
\end{cases}
\label{eq:u_control}
\end{equation}
Here, $f_{i,j}(t)$ and $f_{\mathrm{S},j}(t)$ denote the instantaneous GPU frequencies of 
client $i$ and the edge server, while $P_{i,j}^{\mathrm{TX}}(t)$ and $P_{\mathrm{S},j}^{\mathrm{TX}}(t)$ 
represent their respective uplink and downlink transmit powers. 
The cut-layer index $L_{c,i}$ specifies which subset of DNN layers is 
executed locally and which subset is offloaded to the edge server, thereby influencing the mapping from control inputs 
to the service time $s_{i,j}(\cdot)$.
Moreover, each client and the edge server are subject to bounded GPU frequencies and transmit powers,
\(
f_{i}^{\min} \le f_{i,j} \le f_{i}^{\max}, 
\)
\(f_{\mathrm{S}}^{\min} \le f_{\mathrm{S},j} \le f_{\mathrm{S}}^{\max}, 
\) 
\( P_{i}^{\mathrm{TX},\min} \le P_{i,j}^{\mathrm{TX}} \le P_{i}^{\mathrm{TX},\max}, \)
\(
P_{\mathrm{S}}^{\mathrm{TX},\min} \le P_{\mathrm{S},j}^{\mathrm{TX}} \le P_{\mathrm{S}}^{\mathrm{TX},\max}
\),
which jointly constrain the feasible processing and transmission rates, thereby implying both lower and upper bounds on the resulting service times.

\section{Problem Formulation}
\label{sec:problem_formulation}
Our goal is to jointly optimize the cut-layer decisions $\{L_{c,i}\}$ and the control variables 
$u_{i,j} = \{f_{i,j},P_{i,j}^{\mathrm{TX}},f_{\mathrm{S},j},P_{\mathrm{S},j}^{\mathrm{TX}}\}$
to balance per-round energy consumption and training latency.
Thus, we consider the optimal hybrid control problem with energy-latency tradeoff objective:
\begin{subequations}\label{eq:obj_energy_latency}
\begin{align}
{\textbf{P1}}:\;\min_{\substack{
\{u_{i,j},\,L_{c,i}\}
}}
\quad
& E_{\mathrm{r}}\!\left(u_{i,j},s_{i,j},L_{c,i}\right)
\notag\\
&\qquad+\lambda\, T_{\mathrm{r}}\!\left(u_{i,j},s_{i,j},L_{c,i}\right)
\label{eq:obj_energy_latency_a}\\
\text{s.t. }~&
f_{i}^{\min} \le f_{i,j} \le f_{i}^{\max},
\label{eq:obj_energy_latency_b}\\
&
f_{\mathrm{S}}^{\min} \le f_{\mathrm{S},j} \le f_{\mathrm{S}}^{\max},
\label{eq:obj_energy_latency_c}\\
&
P_{i}^{\mathrm{TX},\min} \le P_{i,j}^{\mathrm{TX}} \le P_{i}^{\mathrm{TX},\max},
\label{eq:obj_energy_latency_d}\\
&
P_{\mathrm{S}}^{\mathrm{TX},\min} \le P_{\mathrm{S},j}^{\mathrm{TX}} \le P_{\mathrm{S}}^{\mathrm{TX},\max}.
\label{eq:obj_energy_latency_e}
\end{align}
\end{subequations}
where $\lambda>0$ controls the tradeoff between energy costs and training latency.

\subsection{Problem Transformation}
By lifting $\{s_{i,j}\}$ as high-level decision variables, we can rewrite $T_{\mathrm r}\left(z_{i,j},u_{i,j},s_{i,j},L_{c,i}\right)$ as $\widehat{T}(\{s_{i,j}\})$.
Specifically, the per-round latency is given by
\begin{align}
&T_{\mathrm r}\!\left(u_{i,j},s_{i,j},L_{c,i}\right)\notag\\
= &\max_{i\in\mathcal{N}}\Big\{s_{i,1}(f_{i,j},L_{c,i})+s_{i,2}(P_{i,j}^{\mathrm{TX}},L_{c,i})\Big\}
+\sum_{i\in\mathcal{N}} s_{i,3}(f_{\mathrm{S},j},\notag\\
&L_{c,i})+\max_{i\in\mathcal{N}}\Big\{s_{i,4}(P_{\mathrm{S},j}^{\mathrm{TX}},L_{c,i})+s_{i,5}(f_{i,j},L_{c,i})\Big\}\notag\\
\triangleq & \widehat{T}\big(\{s_{i,j}(u_{i,j},L_{c,i})\}\big),
\label{eq:Tr_s_form}
\end{align}
where each $s_{i,j}$ is determined by $(u_{i,j},L_{c,i})$ via the stage models.
To handle the $\max(\cdot)$ terms, we introduce epigraph variables $T_1$ and $T_2$ such that
\begin{align}
s_{i,1}+s_{i,2} &\le T_1,\ \forall i\in\mathcal N,\label{eq:epi_T1}\\
s_{i,4}+s_{i,5} &\le T_2,\ \forall i\in\mathcal N.
\label{eq:epi_T2}
\end{align}

Then, minimizing $\lambda\,\widehat{T}(\{s_{i,j}\})$ is equivalent to minimizing
$\lambda(T_1+T_2)+\lambda\sum_{i\in\mathcal N}s_{i,3}$ subject to Eq.~\eqref{eq:epi_T1} and Eq.~\eqref{eq:epi_T2},
because at optimality $T_1=\max_i(s_{i,1}+s_{i,2})$ and $T_2=\max_i(s_{i,4}+s_{i,5})$.

Minimizing the energy cost 
$\psi_{i,j}(u_{i,j},s_{i,j},L_{c,i})$ 
of client~$i$ at stage~$j$ for a given duration 
$s_{i,j}$
constitutes a classical optimal control problem 
with state specified at a fixed terminal time~\cite{bryson2018applied,912126}. 
Let the optimal control input be denoted as $\{\mu_{i,j}^\ast(s_{i,j}),L_{c,i}^*\}$ and the optimal stage cost can be denoted as 
\begin{align}
\Psi_{i,j}(s_{i,j}) =&
\min_{u_{i,j}(t),L_{c,i}} E_{i,j}(u_{i,j},s_{i,j},L_{c,i}) \notag \\
= &\min_{\ell\in\{1,2,..., L_{\max} \}}\min_{u_{i,j}(t)}E_{i,j}\big(u_{i,j},s_{i,j},L_{c,i}{=}\ell\big)  \notag \\
= &\min_{\ell\in\{1,2,..., L_{\max} \}}\min_{u_{i,j}(t)}E_{i,j}^{(\ell)}\big(u_{i,j},s_{i,j}\big)  \notag \\
= &\min_{\ell\in\{1,2,..., L_{\max} \}}\widehat{\Psi}_{i,j}^{(\ell)}\big(s_{i,j}\big) 
\label{eq:stage-opt1}
\end{align}
where  $\ell$ denotes each candidate cut layer and $L_{\max}\leq L$ denotes the maximum admissible cut layer. Thus, the optimal control problem is denoted by
\begin{subequations}
\begin{align}
\textbf{P2}:&
\min_{\substack{\{s_{i,j},y_{i,\ell}\},\\ T_1,\,T_2}}
\sum_{i=1}^{N}\sum_{j=1}^{5}\sum_{\ell=1}^{L_{\max}}
y_{i,\ell}\,\widehat{\Psi}_{i,j}^{(\ell)}(s_{i,j})
+\lambda(T_1+T_2)\notag\\
&+\lambda \sum_{i\in\mathcal{N}} s_{i,3}
\label{eq:P_original_a}
\\
\text{s.t. }
&
s_{i,1}+s_{i,2} \le T_1,\quad \forall i\in\mathcal N,
\label{eq:P_original_b}\\
&
s_{i,4}+s_{i,5} \le T_2,\quad \forall i\in\mathcal N,
\label{eq:P_original_c}\\
&
\sum_{\ell =1}^{L_{\max}} S^{(\ell),\min}_{i,j}\,y_{i,\ell}
\le s_{i,j}
\le \sum_{\ell =1}^{L_{\max}} S^{(\ell),\max}_{i,j}\,y_{i,\ell},
\ \forall i,j,
\label{eq:P_original_d}\\
&
\sum_{\ell=1}^{L_{\max}} y_{i,\ell}=1,\quad
y_{i,\ell}\in\{0,1\},\quad \forall i.
\label{eq:P_original_e}
\end{align}
\end{subequations}
where $y_{i,\ell}$ is a binary cut-layer selection indicator for client $i$, with
$y_{i,\ell}=1$ indicating that client $i$ chooses cut layer $\ell$, and $y_{i,\ell}=0$ otherwise. The
service-time bounds in \eqref{eq:P_original_d} couple this discrete choice with the stage durations
$\{s_{i,j}\}$ by restricting $s_{i,j}$ to the range corresponding to the selected cut layer, and the constraint in \eqref{eq:P_original_e} enforces that each client selects exactly one cut layer.
Moreover, $S^{(\ell),\min}_{i,j}$ and $S^{(\ell),\max}_{i,j}$ are denoted by
\begin{align}
&S^{(\ell),\min}_{i,j}=\notag\\
 &\left\{ 
 \begin{array}{l}
\!\!{\sum_{l=1}^{\ell}\rho_l}/{(n_i f_i^{\max})}, 
 \ j = 1,\ i \in \mathcal{N},\\
{d_i^{\mathrm{U}}(\ell)}/{\bigg(B^{\mathrm{U}}\log_2\!\Bigl(1+\frac{P_i^{\mathrm{TX},\max}h_i^2}{N_0}\Bigr)\bigg)}, 
\ j =2, \ i \in \mathcal{N},\\
\!\!{\sum_{l=\ell+1}^{L}(\rho_l+\varpi_l)}/{(n_{\mathrm S} f_{\mathrm S}^{\max})}, 
\ j =3, \ i \in \mathcal{N}.\\
{d_i^{\mathrm{D}}(\ell)}/{\bigg(B^{\mathrm{D}}\log_2\!\Bigl(1+\frac{P_{\mathrm S}^{\mathrm{TX},\max}h_{\mathrm{S}}^2}{N_0}\Bigr)\bigg)},\ j =4, \ i \in \mathcal{N},\\
\!\! {\sum_{l=1}^{\ell}\varpi_l}/{(n_i f_i^{\max})}
,\ j =5, \ i \in \mathcal{N},\\
\end{array} 
\right.
\label{eq:s_cases_merged_least}
\end{align}
\begin{align}
&S^{(\ell),\max}_{i,j}=\notag\\
 &\left\{ 
 \begin{array}{l}
\!\!{\sum_{l=1}^{\ell}\rho_l}/{(n_i f_i^{\min})}, 
 \ j = 1,\ i \in \mathcal{N},\\
 {d_i^{\mathrm{U}}(\ell)}/{\bigg(B^{\mathrm{U}}\log_2\!\Bigl(1+\frac{P_i^{\mathrm{TX},\min}h_i^2}{N_0}\Bigr)\bigg)}, 
\ j =2, \ i \in \mathcal{N},\\
\!\!{\sum_{l=\ell+1}^{L}(\rho_l+\varpi_l)}/{(n_{\mathrm S} f_{\mathrm S}^{\min})}, 
\ j =3, \ i \in \mathcal{N}.\\
{d_i^{\mathrm{D}}(\ell)}/{\bigg(B^{\mathrm{D}}\log_2\!\Bigl(1+\frac{P_{\mathrm S}^{\mathrm{TX},\min}h_{\mathrm{S}}^2}{N_0}\Bigr)\bigg)},\ j =4, \ i \in \mathcal{N},\\
\!\! {\sum_{l=1}^{\ell}\varpi_l}/{(n_i f_i^{\min})}
,\ j =5, \ i \in \mathcal{N},\\
\end{array} 
\right.
\label{eq:s_cases_merged_least1}
\end{align}


We note that this paper does not explicitly optimize test accuracy. In our prior work~\cite{11049940,10980018,11447454}, we showed that shallower split points generally lead to faster convergence, whereas overly deep split points may degrade learning performance. Motivated by this observation, we restrict the cut-layer design space to
\[
\mathcal L_{\max}\triangleq \{1,2,\ldots,L_{\max}\}.
\]
 Accordingly, the subsequent optimization focuses on the system-level energy--latency tradeoff over the restricted cut-layer set.

To analyze the energy-latency tradeoff, we consider the Pareto efficiency associated with Problem \textbf{P2}. For any feasible decision vector $\xi$, define
\begin{equation}\label{eq:pareto_energy}
E(\xi)\triangleq \sum_{i=1}^{N}\sum_{j=1}^{5}\Psi_{i,j}(s_{i,j}),
\end{equation}
and
\begin{equation}\label{eq:pareto_latency}
T(\xi)\triangleq (T_1+T_2)+\sum_{i\in\mathcal N}s_{i,3},
\end{equation}
where $\{s_{i,j}\}$, $T_1$, and $T_2$ are components of $\xi$.

A feasible solution $\xi^\ast$ is Pareto efficient with respect to $(E,T)$ if there exists no feasible $\hat{\xi}$ such that
$E(\hat{\xi})\le E(\xi^\ast)$ and $T(\hat{\xi})\le T(\xi^\ast)$, with at least one strict inequality. Consider the weighted-sum scalarization
\begin{equation}\label{eq:weighted_sum}
\min_{\xi\in\mathcal X}\; F_{\lambda}(\xi)\triangleq E(\xi)+\lambda T(\xi),\qquad \lambda>0,
\end{equation}
where $\mathcal X$ is the feasible set of \textbf{P2}. For any fixed $\lambda>0$, any global minimizer of \eqref{eq:weighted_sum} is Pareto efficient.
However, varying $\lambda$ does not necessarily recover all Pareto-efficient solutions, because the attainable objective set
\(
\{(E(\xi),T(\xi)):\xi\in\mathcal X\}
\)
is generally non-convex in our setting due to the binary variables $y_{i,\ell}\in\{0,1\}$ and the generally non-convex stage-energy term
$\Psi_{i,j}(s_{i,j})=\min_{\ell}\widehat{\Psi}^{(\ell)}_{i,j}(s_{i,j})$. Hence, \eqref{eq:weighted_sum} may recover only supported Pareto-efficient points.

\section{Special Case: Model-splitting Optimization}
\label{subsec:cutonly}
In this section, we consider model-splitting optimization for SFL, which is an important special case of \textbf{P1}. While the solution space of the integer problem is exponential, we introduce auxiliary variables to divide the joint problem into client subproblems to solve it exactly in polynomial time.
\subsection{Model-splitting Problem Formulation}
We first study a \emph{model-splitting} optimization of the original energy-latency problem in Eq.~\eqref{eq:obj_energy_latency}. 
Specifically, we fix the per-stage
resource controls at prescribed feasible constants, e.g.,
$
P_{i,j}^{\mathrm{TX}}=\hat P_{i,j}^{\mathrm{TX}}$, $
P_{\mathrm{S},j}^{\mathrm{TX}}=\hat P_{\mathrm{S},j}^{\mathrm{TX}}$, $
f_{i,j}=\hat f_{i,j}$, $
f_{\mathrm{S},j}=\hat f_{\mathrm{S},j}$,
where $\hat P_{i,j}^{\mathrm{TX}}$, $\hat P_{\mathrm{S},j}^{\mathrm{TX}}$,
$\hat f_{i,j}$, and $\hat f_{\mathrm{S},j}$ denote prescribed feasible
constants for the client transmit power, server transmit power, client GPU
frequency, and server GPU frequency, respectively.
In this case, for each candidate cut layer ${\ell}\in\{1,\dots,L_{\max}\}$, the corresponding stage time and energy become deterministic constants, denoted by 
\begin{align}
    S_{i,j}^{[{\ell}]}&=s_{i,j}(\hat f_{i,j}, \hat P_{i,j}^{\mathrm{TX}},\hat f_{\mathrm{S},j},\hat P_{\mathrm{S},j}^{\mathrm{TX}},{\ell}),\\
    E_{i,j}^{[{\ell}]}&=E_{i,j}(\hat f_{i,j}, \hat P_{i,j}^{\mathrm{TX}},\hat f_{\mathrm{S},j},\hat P_{\mathrm{S},j}^{\mathrm{TX}},{\ell}).
\end{align}
 Furthermore,
we introduce binary decision variables $y_{i,{\ell}}\in\{0,1\}$ indicating whether client $i$ selects cut layer ${\ell}$, i.e.,
\begin{equation}
\sum_{\ell=1}^{L_{\max}}y_{i,{\ell}}=1,\ y_{i,{\ell}}\in\{0,1\},\quad \forall i\in\mathcal N.
\label{eq:onehot}
\end{equation}
Then the stage time of client $i$ can be rewritten as affine selections:
\begin{align}
s_{i,j}=&\sum_{\ell=1}^{L_{\max}} S_{i,j}^{[\ell]}y_{i,\ell},\quad \forall i,j,\label{eq:affine_time_cutonly}
\end{align}

Substituting \eqref{eq:affine_time_cutonly}  into \eqref{eq:round_latency_parallel}--\eqref{eq:obj_energy_latency}, we obtain
\begin{subequations}
\begin{align}
\textbf{P3}:
\min_{\{y_{i,\ell}\}}
&
\sum_{i\in\mathcal N}\sum_{\ell=1}^{L_{\max}}E_i^{[\ell]}y_{i,\ell}
+\lambda\Big(
\max_{i\in\mathcal N}\sum_{\ell=1}^{L_{\max}}\big(S_{i,1}^{[\ell]}+S_{i,2}^{[\ell]}\big)y_{i,\ell}
\notag\\[-2pt]
&
+\sum_{i\in\mathcal N}\sum_{\ell=1}^{L_{\max}}S_{i,3}^{[\ell]}y_{i,\ell}
+\max_{i\in\mathcal N}\sum_{\ell=1}^{L_{\max}}\big(S_{i,4}^{[\ell]}+S_{i,5}^{[\ell]}\big)y_{i,\ell}
\Big)
\label{eq:Pcut_a}\\
\text{s.t.}\quad
&
\sum_{\ell=1}^{L_{\max}}y_{i,\ell}=1,\quad
y_{i,\ell}\in\{0,1\},\quad
\forall i\in\mathcal N.
\label{eq:Pcut_b}
\end{align}
\end{subequations}
where $E_i^{[\ell]}\triangleq \sum_{j=1}^{5}E_{i,j}^{[\ell]}\ (\forall i,\ \ell)$.
Problem \textbf{P3} in \eqref{eq:Pcut_a} is a 0$-$1 discrete optimization that searches for the best cut layer across clients
with exponentially cut-layer configurations in general. Despite its complexity, we will develop a polynomial-time algorithm to solve it exactly in the subsequent development.
\subsection{Solution Approach}
\label{subsec:cutonly_solution}
We
introduce two variables $T_1$ and $T_2$ satisfying:
\begin{align}
\sum_{\ell=1}^{L_{\max}}\big(S_{i,1}^{[\ell]}+S_{i,2}^{[\ell]}\big)y_{i,\ell}\le T_1,\quad \forall i\in\mathcal N,\\
\sum_{\ell=1}^{L_{\max}}\big(S_{i,4}^{[\ell]}+S_{i,5}^{[\ell]}\big)y_{i,\ell}\le T_2,
\quad \forall i\in\mathcal N.
\label{eq:epigraph_T12}
\end{align}
Then, Problem \textbf{P3} can be equivalently transformed into
\begin{subequations}
\begin{align}
\textbf{P3}^\prime:\;
&\min_{\{y_{i,\ell}\},\,T_1,\,T_2}
\sum_{i\in\mathcal N}\sum_{\ell=1}^{L_{\max}}\Big(E_i^{[\ell]}+\lambda S_{i,3}^{[\ell]}\Big)y_{i,\ell}
+\lambda(T_1+T_2)
\label{eq:Pcut_epigraph_y_noABC_a}\\
\text{s.t.}\quad&
\sum_{\ell=1}^{L_{\max}}y_{i,\ell}=1,\quad
y_{i,\ell}\in\{0,1\},\quad
\forall i\in\mathcal N,
\label{eq:Pcut_epigraph_y_noABC_b}\\
&
\sum_{\ell=1}^{L_{\max}}\big(S_{i,1}^{[\ell]}+S_{i,2}^{[\ell]}\big)y_{i,\ell}\le T_1,
\quad \forall i\in\mathcal N,
\label{eq:Pcut_epigraph_y_noABC_c}\\
&
\sum_{\ell=1}^{L_{\max}}\big(S_{i,4}^{[\ell]}+S_{i,5}^{[\ell]}\big)y_{i,\ell}\le T_2,
\quad \forall i\in\mathcal N.
\label{eq:Pcut_epigraph_y_noABC_d}
\end{align}
\end{subequations}

To obtain the optimal solution, we enumerate the candidate thresholds according to
\begin{align}
T_1\in&\Big\{S_{i,1}^{[\ell]}+S_{i,2}^{[\ell]}\Big\},\\
T_2\in&\Big\{S_{i,4}^{[\ell]}+S_{i,5}^{[\ell]}\Big\},
\label{eq:threshold_candidates_T12}
\end{align}
which are finite sets of cardinality at most $N(L-1)$.
For any fixed $(T_1,T_2)$, client $i$ selects the  cut layer with minimum cost:
\begin{equation}
\ell_i^\ast(T_1,T_2)
=
\arg \min_{\ell\in\mathcal F_i(T_1,T_2)}
\Big(
E_i^{[\ell]}+\lambda S_{i,3}^{[\ell]}
\Big),
\label{eq:inner_choice}
\end{equation}
where the feasible cut-layer set $\mathcal F_i(T_1,T_2)$ is defined as
\begin{equation}
\mathcal F_i(T_1,T_2)\triangleq
\Big\{\ell:\ 
S_{i,1}^{[\ell]}+S_{i,2}^{[\ell]}\le T_1,\ 
S_{i,4}^{[\ell]}+S_{i,5}^{[\ell]}\le T_2
\Big\}.
\label{eq:feasible_set_T12}
\end{equation}

The resulting objective value is
\begin{equation}
J(T_1,T_2)
=
\underbrace{\sum_{i\in\mathcal N}
 \min_{\ell\in\mathcal F_i(T_1,T_2)}
\Big(
E_i^{[\ell]}+\lambda S_{i,3}^{[\ell]}
\Big)}_{G(T_1, T_2):\ \text{Non-increasing}}
+\underbrace{\lambda(T_1+T_2)} _{\substack{H(T_1, T_2): \ \text{Strictly} \\ \text{increasing}}}.
\label{eq:threshold_obj}
\end{equation}
Finally, the globally optimal thresholds are obtained by
\begin{equation}
(T_1^{\ast},T_2^{\ast})
=
\arg\min_{T_1,T_2} J(T_1,T_2),
\label{eq:best_threshold}
\end{equation}
which yields a globally optimal cut-layer decision $\{y_{i,\ell}^\ast\}$ for \eqref{eq:Pcut_a}.

\begin{algorithm}[t]
\caption{Optimal Model-splitting Optimization via Plane Sweep and Pruning}
\label{alg:sweep_pruning}
\KwIn{$\{S_{i,j}^{[\ell]}\},\{E_i^{[\ell]}\},\lambda$. \ $V_i^{[\ell]}\triangleq E_i^{[\ell]}+\lambda S_{i,3}^{[\ell]}$.}
\KwOut{$J^\ast,\{\ell_i^\ast\}$.}

$\mathcal T_1\gets \mathrm{sort}(\{S_{i,1}^{[\ell]}+S_{i,2}^{[\ell]}\})$,  
$\mathcal T_2\gets \mathrm{sort}(\{S_{i,4}^{[\ell]}+S_{i,5}^{[\ell]}\})$\; 
$G_{\min}\gets \sum_{i\in\mathcal N}\min_{\ell\in\{1,\dots,L_{\max}\}} (E_i^{[\ell]}+\lambda S_{i,3}^{[\ell]})$\;
$\mathcal B[t]\gets \{(i,\ell)\mid S_{i,1}^{[\ell]}+S_{i,2}^{[\ell]}=t\}$; $J^\ast\gets \infty$\;

\ForEach{$T_2\in\mathcal T_2$}{\tcp{outer loop: fix $T_2$}
    $c_i\gets\infty,\ \ell_i\gets-1\ (\forall i)$; $G\gets0$; $\text{count}\gets0$\;
    \For{$k\gets1$ \KwTo $|\mathcal T_1|$}{\tcp{inner loop: sweep $T_1$}

        $T_1\gets \mathcal T_1[k]$\;
        \ForEach{$(i,\ell)\in\mathcal B[T_1]$}{\tcp{process pairs introduced by the current $T_1$}

            \If{$S_{i,4}^{[\ell]}+S_{i,5}^{[\ell]}\le T_2 \ \text{and}\ (E_i^{[\ell]}+\lambda S_{i,3}^{[\ell]})<c_i$}{ \tcp{passes $T_2$ check and reduces client-$i$ cost}

                \eIf{$c_i=\infty$}{$\text{count}\gets \text{count}+1$; $G\gets G+(E_i^{[\ell]}+\lambda S_{i,3}^{[\ell]})$;}
                {\tcp{Replace the previous best cut of client $i$}$G\gets G-c_i+(E_i^{[\ell]}+\lambda S_{i,3}^{[\ell]})$;}
                $c_i\gets (E_i^{[\ell]}+\lambda S_{i,3}^{[\ell]})$; $\ell_i\gets \ell$\;
            }
        }
        \If{$\text{count}=|\mathcal N|$}{
            $J\gets G+\lambda(T_1+T_2)$\;
            \If{$J<J^\ast$}{$J^\ast\gets J;\ \{\ell_i^\ast\}\gets\{\ell_i\}$}
            \tcp{Prune the remaining $T_1$ candidates for this fixed $T_2$ (\textbf{Theorem~\ref{thm:global_opt}})}
            \If{$k<|\mathcal T_1|\ \text{and}\ \lambda(\mathcal T_1[k+1]-\mathcal T_1[k])>G-G_{\min}$}{\textbf{break}}
        }
    }
}
\Return{$J^\ast,\{\ell_i^\ast\}$}\;
\end{algorithm}

To solve Problem \textbf{P3}, \textbf{Algorithm~\ref{alg:sweep_pruning}} first constructs the candidate threshold sets
\[
\begin{aligned}
&\mathcal{T}_1 \triangleq \Big\{\, S^{[\ell]}_{i,1}+S^{[\ell]}_{i,2}\; \big|\; i\in\mathcal N,\ \ell\in\{1,\dots,L_{\max}\}\,\Big\},
\\
&\mathcal{T}_2 \triangleq \Big\{\, S^{[\ell]}_{i,4}+S^{[\ell]}_{i,5}\; \big|\; i\in\mathcal N,\ \ell\in\{1,\dots,L_{\max}\}\,\Big\}.
\end{aligned}
\]
The two sets are sorted in ascending order and enumerated as candidate threshold values.
For each fixed $T_2\in\mathcal{T}_2$, the algorithm sequentially sweeps $T_1\in\mathcal{T}_1$,
updates the best feasible cut-layer choice for each client, and applies a bound-based pruning rule
to terminate the $T_1$ scan early whenever no better objective value can be attained.
By enumerating all candidate threshold pairs together with this pruning rule,
\textbf{Algorithm~\ref{alg:sweep_pruning}} solves Problem \textbf{P3} to global optimality.

\begin{theorem}[Optimality of \textbf{Algorithm~\ref{alg:sweep_pruning}}]
\label{thm:global_opt}
\textbf{Algorithm~\ref{alg:sweep_pruning}} obtains the globally optimal solution to Problem \textbf{P3}.
\end{theorem}

\begin{proof}
Fix any $T_2 \in \mathcal T_2$. Let $\{\tau_{1,k}\}_{k=1}^{|\mathcal T_1|}$ denote
the elements of $\mathcal T_1$ sorted in ascending order, i.e.,
$\tau_{1,1} \le \tau_{1,2}  \le\cdots \le\tau_{1,|\mathcal T_1|}$.
Recall $J(T_1,T_2)=G(T_1,T_2)+\lambda(T_1+T_2)$, where $G(\cdot,T_2)$ is
non-increasing in $T_1$ and is lower bounded by
$G_{\min}\triangleq\sum_{i\in\mathcal N}\min_{\ell\in\{1,\dots,L_{\max}\}} (E_i^{[\ell]}+\lambda S_{i,3}^{[\ell]})$.

At step $k$, for any future candidate $\tau_{1,m}>\tau_{1,k}$, we have
\[
\begin{aligned}
&J(\tau_{1,m},T_2)-J(\tau_{1,k},T_2)\\
&= \big(G(\tau_{1,m},T_2)-G(\tau_{1,k},T_2)\big)
  +\lambda(\tau_{1,m}-\tau_{1,k}) \\
&\ge (G_{\min}-G(\tau_{1,k},T_2)) + \lambda(\tau_{1,m}-\tau_{1,k}),
\end{aligned}
\]
where the inequality uses $G(\tau_{1,m},T_2)\ge G_{\min}$.
Since $\tau_{1,m}-\tau_{1,k}\ge \tau_{1,k+1}-\tau_{1,k}$, it follows that
\[
J(\tau_{1,m},T_2)-J(\tau_{1,k},T_2)
\ge -\Delta G_{\max} + \lambda(\tau_{1,k+1}-\tau_{1,k}),
\]
with $\Delta G_{\max}\triangleq G(\tau_{1,k},T_2)-G_{\min}$.
Thus, if $\lambda(\tau_{1,k+1}-\tau_{1,k})>\Delta G_{\max}$, then
$J(\tau_{1,m},T_2)>J(\tau_{1,k},T_2)$ for all $m>k$,
so no better solution exists beyond $\tau_{1,k}$ and the scan over $T_1$
can be safely terminated for this fixed $T_2$.
Furthermore,
enumerating all $T_2\in\mathcal T_2$ ensures that the global minimizer
$(T_1^\ast,T_2^\ast)$ is not missed, hence the algorithm returns a
globally optimal $\{y_{i,\ell}^\ast\}$.
\end{proof}

\textbf{Complexity.}
\textbf{Algorithm~\ref{alg:sweep_pruning}} first constructs the candidate sets
$\mathcal T_1=\{S^{[\ell]}_{i,1}+S^{[\ell]}_{i,2}\}$ and
$\mathcal T_2=\{S^{[\ell]}_{i,4}+S^{[\ell]}_{i,5}\}$ and sorts them, which costs
$O(NL\log(NL))$ since $|\mathcal T_1|,|\mathcal T_2|\le N(L-1)$.
For each fixed $T_2\in\mathcal T_2$, Algorithm~1 scans $T_1\in\mathcal T_1$ in ascending
order. Because $T_1$ is enumerated over the sorted distinct values in $\mathcal T_1$,
the uplink-feasible set $\{(i,\ell)\mid S^{[\ell]}_{i,1}+S^{[\ell]}_{i,2}\le T_1\}$
expands only when $T_1$ advances to the next candidate; hence, at a given $T_1$ it suffices
to process only the newly activated pairs $(i,\ell)$ satisfying
$S^{[\ell]}_{i,1}+S^{[\ell]}_{i,2}=T_1$.
For each such pair, the algorithm performs (i) one comparison to test downlink feasibility
$S^{[\ell]}_{i,4}+S^{[\ell]}_{i,5}\le T_2$, (ii) one evaluation of
$E_i^{[\ell]}+\lambda S_{i,3}^{[\ell]}$, and (iii) a constant-time comparison and possible update
of the stored best value $c_i$ together with the running sum. All these operations take
$O(1)$ time per processed pair.
Thus, over a complete scan of $T_1$ for a fixed $T_2$, each $(i,\ell)$ is processed at
most once, and the total update cost per $T_2$ is $O(NL)$.
Consequently, the overall running time of \textbf{Algorithm~\ref{alg:sweep_pruning}} is
\begin{equation}
\begin{aligned}
&O\big(|\mathcal T_2|\cdot NL\big)+O\big(NL\log(NL)\big)
 \\ =&O\big(N^2L^2+NL\log(NL)\big)\\=& O\big(N^2L^2\big).
\end{aligned}
\end{equation}

\section{Joint Optimization of Model Splitting and Resource Allocation}\label{solu_appro}
We now present a polynomial-time approximation method for the general joint optimization of model splitting and resource allocation. Specifically, we first reformulate Problem $\textbf{P2}$ exactly as a two-dimensional master problem with respect to the synchronization variables $(T_1, T_2)$. Since the resulting master problem is generally nonconvex, we then approximate it via a uniform grid search over the $(T_1, T_2)$ domain. This leads to a polynomial-time approximation algorithm with an explicit performance guarantee.
\subsection{{Two-Dimensional Bottleneck Projection}}
Recall that the latency term in Problem $\textbf{P2}$ is characterized by the two epigraph variables $T_1$ and $T_2$, where
\begin{align}
    s_{i,1}+s_{i,2}\le T_1,\qquad \forall i\in\mathcal N,\\
s_{i,4}+s_{i,5}\le T_2,\qquad \forall i\in\mathcal N.
\end{align}
Once $(T_1,T_2)$ is fixed, the optimization variables across clients become decoupled. This key observation enables us to project the original mixed-integer problem onto the two synchronization variables.
Specifically,
for each client $i$, stage $j$, 
and cut layer $\ell\in\{1,2,\ldots,L_{\max}\}$, let $\widehat{\Psi}_{i,j}^{(\ell)}(s_{i,j})$ denote the minimum stage energy when the service time is $s_{i,j}$ and the cut layer is fixed at $\ell$, i.e.,
\begin{equation}
\widehat{\Psi}_{i,j}^{(\ell)}(s_{i,j})=\min_{u_{i,j}(t)}E_{i,j}\big(u_{i,j},s_{i,j},L_{c,i}{=}\ell\big).
\end{equation}
Moreover, the feasible service-time is given by
\begin{equation}
S_{i,j}^{(\ell),\min}\le s_{i,j}\le S_{i,j}^{(\ell),\max}.
\label{eq:sec6_stage_bound}
\end{equation}

For fixed $(T_1,T_2)$, the per-client optimization decomposes into three independent subproblems.
First, the server-side FP/BP stage contributes additively to the objective and is not directly constrained by $T_1$ or $T_2$. Accordingly, we define
\begin{subequations}
\begin{align}
\textbf{P4}^\prime:\;U_{i,3}^{(\ell)}&
\triangleq
\min_{s_{i,3}}\;
\widehat{\Psi}_{i,3}^{(\ell)}(s_{i,3})+\lambda s_{i,3}
\\
&\text{s.t.}\quad
S_{i,3}^{(\ell),\min}\le s_{i,3}\le S_{i,3}^{(\ell),\max}.
\label{eq:Ui3_sec6}
\end{align}
\end{subequations}

Second, the first synchronization variable $T_1$ constrains the sum of client-side FP time and uplink transmission time, which is given by
\begin{subequations}\label{eq:Ui12_sec6}
\begin{align}
\textbf{P4}^{\prime\prime}:\;U_{i,12}^{(\ell)}(T_1)
\triangleq\;
&\min_{s_{i,1},s_{i,2}}
\ \widehat{\Psi}_{i,1}^{(\ell)}(s_{i,1})+\widehat{\Psi}_{i,2}^{(\ell)}(s_{i,2})
\label{eq:Ui12_sec6_a}\\
\text{s.t.}\quad
&s_{i,1}+s_{i,2}\le T_1,
\label{eq:Ui12_sec6_b}\\
&S_{i,1}^{(\ell),\min}\le s_{i,1}\le S_{i,1}^{(\ell),\max},
\label{eq:Ui12_sec6_c}\\
&S_{i,2}^{(\ell),\min}\le s_{i,2}\le S_{i,2}^{(\ell),\max}.
\label{eq:Ui12_sec6_d}
\end{align}
\end{subequations}
If
\(\
T_1<S_{i,1}^{(\ell),\min}+S_{i,2}^{(\ell),\min},
\)
then cut layer $\ell$ is infeasible for client $i$ under $T_1$, and we set
$U_{i,12}^{(\ell)}(T_1)=+\infty$.

Third,  considering $T_2$ constrains the sum of the downlink transmission time and client-side BP time, we define
\begin{subequations}\label{eq:Ui45_sec6}
\begin{align}
\textbf{P4}^{\prime\prime\prime}:\;U_{i,45}^{(\ell)}(T_2)
\triangleq\;
&\min_{s_{i,4},s_{i,5}}
\ \widehat{\Psi}_{i,4}^{(\ell)}(s_{i,4})+\widehat{\Psi}_{i,5}^{(\ell)}(s_{i,5})
\label{eq:Ui45_sec6_a}\\
\text{s.t.}\quad
&s_{i,4}+s_{i,5}\le T_2,
\label{eq:Ui45_sec6_b}\\
&S_{i,4}^{(\ell),\min}\le s_{i,4}\le S_{i,4}^{(\ell),\max},
\label{eq:Ui45_sec6_c}\\
&S_{i,5}^{(\ell),\min}\le s_{i,5}\le S_{i,5}^{(\ell),\max}.
\label{eq:Ui45_sec6_d}
\end{align}
\end{subequations}
If
\(
T_2<S_{i,4}^{(\ell),\min}+S_{i,5}^{(\ell),\min},
\)
then cut layer $\ell$ is infeasible for client $i$ under $T_2$, and we set
$U_{i,45}^{(\ell)}(T_2)=+\infty$.

For fixed $(T_1,T_2)$, client $i$ evaluates all candidate cut layers and chooses the least-cost feasible one. This yields the client value function
\begin{equation}
V_i(T_1,T_2)
\triangleq
\min_{\ell\in\{1,2,\ldots,L_{\max}\}}
\Big(
U_{i,12}^{(\ell)}(T_1)
+
U_{i,45}^{(\ell)}(T_2)
+
U_{i,3}^{(\ell)}
\Big).
\label{eq:Vi_sec6}
\end{equation}

Hence, the original mixed-integer Problem $\textbf{P2}$ is exactly reduced to the following two-dimensional master problem:
\begin{align}
\textbf{P5}:\quad
\min_{T_1,T_2}\;&
J(T_1,T_2)
\triangleq
\lambda(T_1+T_2)+\sum_{i\in\mathcal N}V_i(T_1,T_2). \label{eq:Pmaster_sec6}
\end{align}
where $V_i(T_1,T_2)$ is defined in \eqref{eq:Vi_sec6}. 
The feasibility conditions are already incorporated into the definitions of 
$U_{i,3}^{(\ell)}$, $U_{i,12}^{(\ell)}(T_1)$, and $U_{i,45}^{(\ell)}(T_2)$, 
with infeasible choices assigned value $+\infty$.

\begin{theorem}[Problem Equivalence]
\label{thm:projection_exact_sec6}
{
Problem $\textbf{P2}$ is equivalent to the master problem in \eqref{eq:Pmaster_sec6}. In particular, solving Problem \textbf{P5} yields an optimal solution to Problem $\textbf{P2}$.
}
\end{theorem}

\begin{proof}
Fix any feasible pair $(T_1,T_2)$. Under this pair, the only coupling among clients is removed, and the optimization variables of different clients become separable. For each client $i$ and each cut layer $\ell$, the optimal contribution to the objective decomposes into the three blocks $U_{i,12}^{(\ell)}(T_1)$,
$U_{i,45}^{(\ell)}(T_2)$ and 
$U_{i,3}^{(\ell)}$. Minimizing over $\ell\in\{1,2,\ldots,L_{\max}\}$ gives the client value function $V_i(T_1,T_2)$ in \eqref{eq:Vi_sec6}. Summing the optimal client values over all clients and adding the synchronization penalty $\lambda(T_1+T_2)$ yields \eqref{eq:Pmaster_sec6}. Thus, minimizing Problem \textbf{P5} over $(T_1,T_2)$ is equivalent to solving Problem $\textbf{P2}$.
\end{proof}

\subsection{Grid Approximation of the Master Problem}
The reformulation in \eqref{eq:Pmaster_sec6} is equivalent to the original problem. However, 
 the master objective $J(T_1,T_2)$ is generally nonconvex because each client value function $V_i(T_1,T_2)$ contains a minimization over discrete cut layers. Since the master problem is only two-dimensional, we approximate it via a uniform grid search over the bounded rectangle
\begin{equation}
T_1\in[\underline{T}_1,\overline{T}_1],
\qquad
T_2\in[\underline{T}_2,\overline{T}_2],
\label{eq:grid_region_sec6}
\end{equation}
where the search box is chosen to contain at least one global optimizer of \eqref{eq:Pmaster_sec6}.
For a grid step size $\Delta>0$, we evaluate $J(T_1,T_2)$ at all grid points in \eqref{eq:grid_region_sec6} and return the best value, denoted by $J_{\mathrm{grid}}$.
The following theorem provides an additive approximation guarantee.

\begin{theorem}[Relative Error Bound]
\label{theorem:relative_error_sec6}
Assume that the search region 
$[\underline{T}_1,\overline{T}_1]\times[\underline{T}_2,\overline{T}_2]$
contains at least one global optimizer $(T_1^\ast,T_2^\ast)$ of
\eqref{eq:Pmaster_sec6}, and let
\(
J^\ast \triangleq J(T_1^\ast,T_2^\ast).
\)
Suppose that
\(
J^\ast \ge \lambda(\underline T_1+\underline T_2),
\)
where $\underline T_1+\underline T_2>0$. If a uniform grid with step size $\Delta$ is used in both dimensions, then
\[
\frac{J_{\mathrm{grid}}-J^\ast}{J^\ast}
\le
\frac{2\Delta}{\underline T_1+\underline T_2}.
\]
Consequently, in order to guarantee
\(
J_{\mathrm{grid}}\le (1+\epsilon)J^\ast,
\)
it suffices to choose
\(
\Delta \le \frac{\epsilon}{2}(\underline T_1+\underline T_2).
\)
\end{theorem}

\begin{proof}
Since the grid spacing is $\Delta$, there exists a grid point $(\widetilde T_1,\widetilde T_2)$ such that
\begin{align}
T_1^\ast\le \widetilde T_1\le T_1^\ast+\Delta, \notag
\\
T_2^\ast\le \widetilde T_2\le T_2^\ast+\Delta.  \notag
\end{align}
For each client $i\in\mathcal N$, increasing $T_1$ or $T_2$ relaxes the corresponding synchronization constraints and thus cannot increase the value function $V_i(T_1,T_2)$. Hence,
\[
V_i(\widetilde T_1,\widetilde T_2)\le V_i(T_1^\ast,T_2^\ast),\qquad \forall i\in\mathcal N.
\]
It follows that
\[
\begin{aligned}
J(\widetilde T_1,\widetilde T_2)
&=
\lambda(\widetilde T_1+\widetilde T_2)+\sum_{i\in\mathcal N}V_i(\widetilde T_1,\widetilde T_2)\\
&\le
\lambda(T_1^\ast+\Delta+T_2^\ast+\Delta)+\sum_{i\in\mathcal N}V_i(T_1^\ast,T_2^\ast)\\
&=
J^\ast+2\lambda\Delta.
\end{aligned}
\]
Since the grid-search algorithm returns the best grid point, we further have
\[
J_{\mathrm{grid}}\le J(\widetilde T_1,\widetilde T_2).
\]

Then, we have
\(
J_{\mathrm{grid}}-J^\ast \le 2\lambda\Delta.
\)
Dividing both sides by $J^\ast$ and using
\(
J^\ast \ge \lambda(\underline T_1+\underline T_2),
\)
we obtain
\[
\frac{J_{\mathrm{grid}}-J^\ast}{J^\ast}
\le
\frac{2\lambda\Delta}{J^\ast}
\le
\frac{2\Delta}{\underline T_1+\underline T_2}.
\]
The desired \((1+\epsilon)\)-approximation guarantee is obtained by requiring
\(
2\Delta/(\underline T_1+\underline T_2)\le \epsilon,
\)
which completes the proof.
\end{proof}


\begin{algorithm}[t]
\caption{Grid Approximation for the Two-Dimensional Master Problem}
\label{alg:grid_projection}
\KwIn{$\mathcal N,\mathcal L,[\underline T_1,\overline T_1],[\underline T_2,\overline T_2],\Delta,\lambda$.}
\KwOut{$J^\ast,(T_1^\ast,T_2^\ast),\{\ell_i^\ast\}$.}

$\mathcal G_1\gets \{\underline T_1,\underline T_1+\Delta,\underline T_1+2\Delta,\dots,\overline T_1\}$\;
$\mathcal G_2\gets \{\underline T_2,\underline T_2+\Delta,\underline T_2+2\Delta,\dots,\overline T_2\}$\;
$J^\ast\gets \infty$\;

\ForEach{$T_1\in\mathcal G_1$}{
    \ForEach{$T_2\in\mathcal G_2$}{
        $J\gets \lambda(T_1+T_2)$\;
        feasible $\gets 1$\;
        Compute $U_{i,3}^{(\ell)}$, $U_{i,12}^{(\ell)}(T_1)$, and $U_{i,45}^{(\ell)}(T_2)$ for all required $(i,\ell,T_1,T_2)$ combinations\;
        \ForEach{$i\in\mathcal N$}{
            $V_i\gets \infty,\ \ell_i\gets -1$\;
            \ForEach{$\ell\in\mathcal L$}{
                \If{$T_1<S_{i,1}^{(\ell),\min}+S_{i,2}^{(\ell),\min}$ \textbf{or} $T_2<S_{i,4}^{(\ell),\min}+S_{i,5}^{(\ell),\min}$}{
                    \textbf{continue}\;
                }
                $C_i^{(\ell)}\gets U_{i,3}^{(\ell)}+U_{i,12}^{(\ell)}(T_1)+U_{i,45}^{(\ell)}(T_2)$\;
                \If{$C_i^{(\ell)}<V_i$}{
                    $V_i\gets C_i^{(\ell)},\ \ell_i\gets \ell$\;
                }
            }
            \If{$V_i=\infty$}{
                feasible $\gets 0$\;
                \textbf{break}\;
            }
            $J\gets J+V_i$\;
        }
        \If{feasible $=1$ \textbf{and} $J<J^\ast$}{
            $J^\ast\gets J,\ (T_1^\ast,T_2^\ast)\gets (T_1,T_2),\ \{\ell_i^\ast\}\gets \{\ell_i\}$\;
        }
    }
}
 \Return{$J^\ast,(T_1^\ast,T_2^\ast),\{\ell_i^\ast\}$}\;
\end{algorithm}



\begin{theorem}[Complexity of Grid Approximation]
\label{thm:grid_complexity_sec6}
Let $\eta>0$ denote the numerical accuracy  to solve the inner subproblems. Then, the proposed grid-search algorithm has total complexity
\(
O\!\big(G_1G_2\,NL\log(1/\eta)\big).
\)
Under bounded search ranges, choosing $\Delta$ according to \textbf{Theorem~\ref{theorem:relative_error_sec6}} yields the total complexity
\(
O\!\big((NL/\epsilon^2)\log(1/\eta)\big).
\)
Hence, \textbf{Algorithm \ref{alg:grid_projection}} runs in polynomial time.
\end{theorem}

\begin{proof}
\textit{1) Complexity of the inner subproblems.}
Fix any grid point $(T_1,T_2)$, client $i$, and candidate cut layer $\ell$. 
For the corresponding client $i$ pair $(i,\ell)$ with cut layer $\ell$, the algorithm evaluates the three inner quantities
\(
U_{i,3}^{(\ell)},\)
\(U_{i,12}^{(\ell)}(T_1),\)
\(U_{i,45}^{(\ell)}(T_2).
\)
We first show that these inner subproblems are convex. For the computation stages $j\in\{1,3,5\}$, under a fixed cut layer $\ell$, the stage-energy functions are given by
\[
\widehat{\Psi}_{i,j}^{(\ell)}(s_{i,j})
=
{a_{i,j}^{(\ell)}}/{s_{i,j}^{2}},
\qquad j\in\{1,3,5\},
\]
where
\[
a_{i,j}^{(\ell)}
=
\begin{cases}
{G_i\left(\sum\limits_{k=1}^{\ell}\rho_k\right)^3}/{n_i^3}, & j=1,\\[1.2ex]
{G_{\mathrm S}\left(\sum\limits_{k=\ell+1}^{L}(\rho_k+\varpi_k)\right)^3}/{n_{\mathrm S}^3}, & j=3,\\[1.2ex]
{G_i\left(\sum\limits_{k=1}^{\ell}\varpi_k\right)^3}/{n_i^3}, & j=5.
\end{cases}
\]
Since
\[
\frac{d^2}{d s_{i,j}^2}\!\left(\frac{a_{i,j}^{(\ell)}}{s_{i,j}^2}\right)
=
\frac{6a_{i,j}^{(\ell)}}{s_{i,j}^4}>0,
\qquad \forall s_{i,j}>0,
\]
it follows that $\widehat{\Psi}_{i,j}^{(\ell)}(s_{i,j})$ is convex on $s_{i,j}>0$.

For the communication stages $j\in\{2,4\}$, under a fixed cut layer $\ell$, the stage-energy functions can be written as
\[
\widehat{\Psi}_{i,j}^{(\ell)}(s_{i,j})
=
b_{i,j}^{(\ell)}\,s_{i,j}\!\left(2^{c_{i,j}^{(\ell)}/s_{i,j}}-1\right),
\qquad j\in\{2,4\},
\]
where
\[
b_{i,j}^{(\ell)}
=
\begin{cases}
\dfrac{B^{\mathrm U}N_0}{h_i^2}, & j=2,\\[1.2ex]
\dfrac{B^{\mathrm D}N_0}{h_{\mathrm S}^2}, & j=4,
\end{cases}
\qquad
c_{i,j}^{(\ell)}
=
\begin{cases}
\dfrac{d_i^{\mathrm U}(\ell)}{B^{\mathrm U}}, & j=2,\\[1.2ex]
\dfrac{d_i^{\mathrm D}(\ell)}{B^{\mathrm D}}, & j=4.
\end{cases}
\]
Define
\[
\phi_{i,j}^{(\ell)}(s_{i,j})
\triangleq
b_{i,j}^{(\ell)}\,s_{i,j}\!\left(2^{c_{i,j}^{(\ell)}/s_{i,j}}-1\right),
\qquad s_{i,j}>0.
\]
Then
\[
\frac{d^2}{d s_{i,j}^2}\phi_{i,j}^{(\ell)}(s_{i,j})
=
\frac{b_{i,j}^{(\ell)}\,2^{c_{i,j}^{(\ell)}/s_{i,j}}
\big(c_{i,j}^{(\ell)}\big)^2(\ln 2)^2}{s_{i,j}^3}
>0,
\ \forall s_{i,j}>0.
\]
Hence, $\widehat{\Psi}_{i,j}^{(\ell)}(s_{i,j})$ is convex for $j\in\{2,4\}$.

Thus, $U_{i,3}^{(\ell)}$ is a one-dimensional convex minimization over the  interval
\(
[S_{i,3}^{(\ell),\min},\,S_{i,3}^{(\ell),\max}].
\)
Moreover, $U_{i,12}^{(\ell)}(T_1)$ and $U_{i,45}^{(\ell)}(T_2)$ are fixed-dimensional convex programs with box constraints and one affine coupling constraint. Hence, these inner subproblems can be solved efficiently. In particular, $U_{i,3}^{(\ell)}$ admits a closed-form solution after projecting the unconstrained minimizer onto its feasible interval, while $U_{i,12}^{(\ell)}(T_1)$ and $U_{i,45}^{(\ell)}(T_2)$ can be solved to accuracy $\eta$ in time polynomial in $\log(1/\eta)$ by standard convex optimization methods. Thus, the total cost per client--cut-layer pair is $O(\log(1/\eta))$.

\textit{2) Complexity of the outer grid search.}
Let
\[
G_1\triangleq \left\lceil\frac{\overline T_1-\underline T_1}{\Delta}\right\rceil+1,
\qquad
G_2\triangleq \left\lceil\frac{\overline T_2-\underline T_2}{\Delta}\right\rceil+1
\]
denote the numbers of grid points along the two dimensions. Then, the grid contains $G_1G_2$ points in total. At each grid point, the algorithm evaluates all $N$ clients and at most $L-1$ candidate cut layers. Combining the complexity of the inner subproblems, the total complexity is
\[
O\!\big(G_1G_2\,N(L-1)\log(1/\eta)\big)
=
O\!\big(G_1G_2\,NL\log(1/\eta)\big).
\]


Finally, under bounded search ranges, \textbf{Theorem~\ref{theorem:relative_error_sec6}} implies that $\Delta=O(\epsilon)$. Hence,
\(
G_1=O(1/\epsilon),\)
\(G_2=O(1/\epsilon),
\)
and thus
\(
G_1G_2=O(1/\epsilon^2).
\)
Substituting this into the above bound yields
\(
O\!\big((NL/\epsilon^2)\log(1/\eta)\big)
\), which completes the proof.
\end{proof}

\subsection{Recovering the Optimal Control Inputs}
Once the optimal solution of the high-level problem $\textbf{P}$ is obtained, the optimal solution to the original hybrid control problem can be recovered by solving the lower-level control subproblems stage by stage. In particular, the high-level optimizer determines the optimal cut-layer decisions $\{L_{c,i}^{*}\}$ and the optimal service-time variables $\{s_{i,j}^{*}\}$. With these decisions solved, the corresponding optimal control input is recovered as
\begin{equation}
u_{i,j}^{*}(s_{i,j}^{*})
=
\arg\min_{u_{i,j}}
\psi_{i,j}\!\left(z_{i,j},u_{i,j},s_{i,j}^{*},L_{c,i}^{*}\right),
\end{equation}
subject to the associated system dynamics and boundary conditions.  Thus, the optimal policy of the original hybrid control problem is obtained by combining the high-level optimal resource-allocation decisions with the recovered low-level control inputs.

  \begin{figure*}[t!]
\subfigure[\centering Cost function with test accuracy (MNIST, IID, ResNet-50).] {
\centering\includegraphics[width=4.1cm]{ 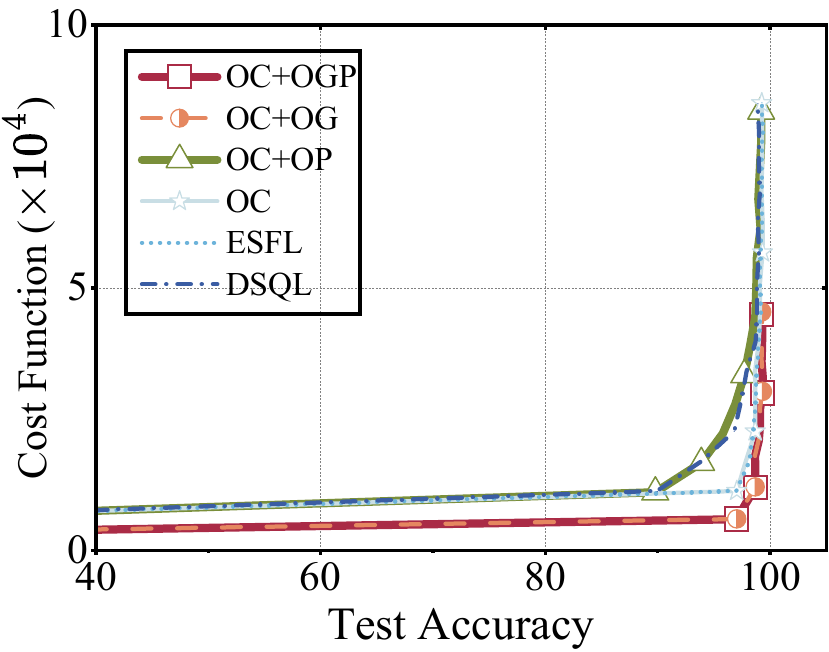}
	}
\subfigure[\centering Cost function with test accuracy (MNIST, non-IID, ResNet-50).] {
\centering\includegraphics[width=4.1cm]{ 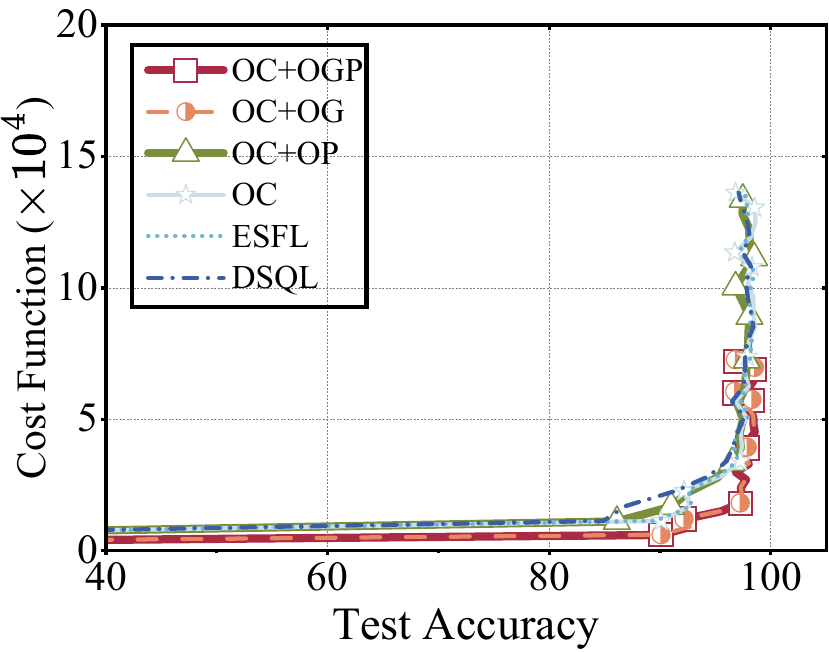}
	}
    \subfigure[\centering Cost function with test accuracy (CIFAR-10, IID, ResNet-50).] {
\centering\includegraphics[width=4.1cm]{ 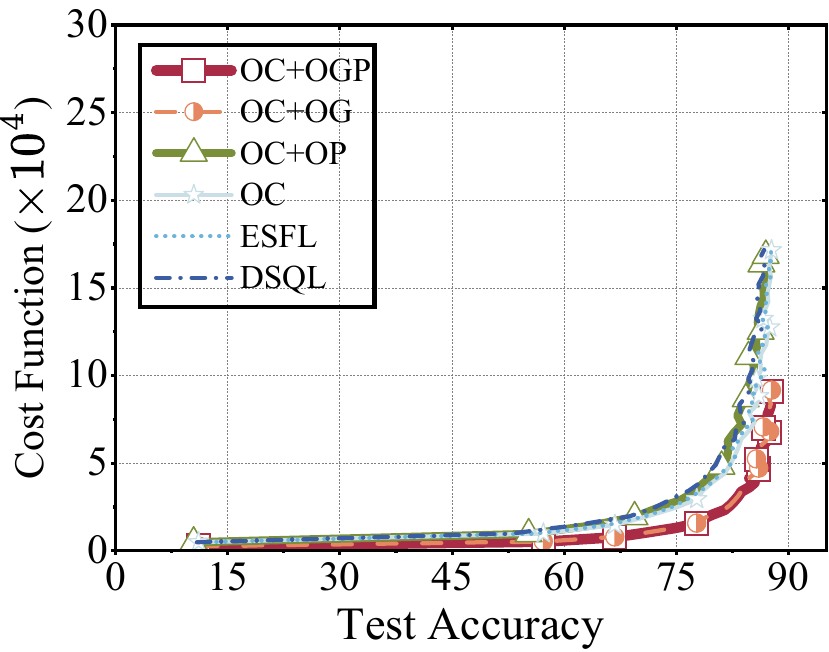}
	}
    \subfigure[\centering Cost function with test accuracy (CIFAR-10, non-IID, ResNet-50).] {
\centering\includegraphics[width=4.1cm]{ 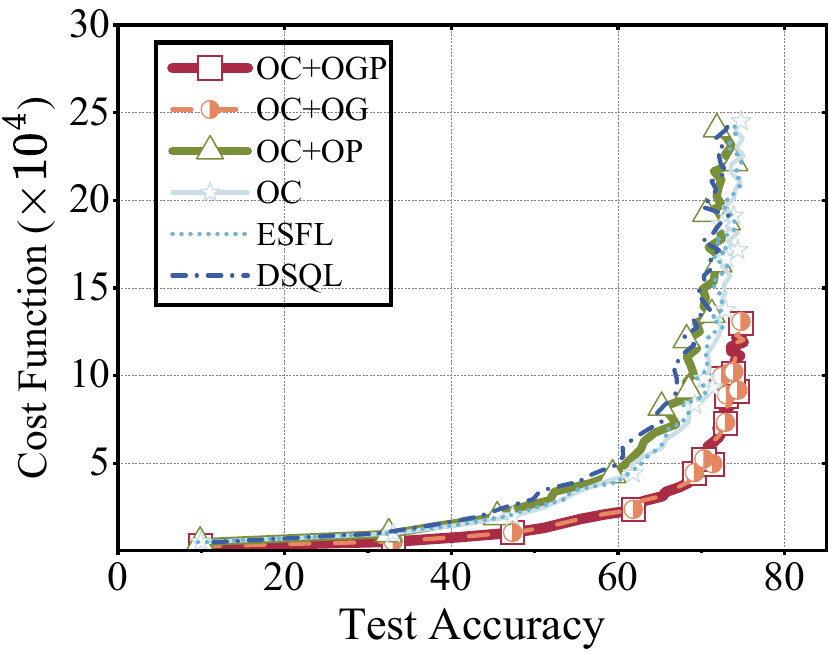}
	}
    \subfigure[\centering Cost function with test accuracy (MNIST, IID, VGG-16).] {
\centering\includegraphics[width=4.1cm]{ 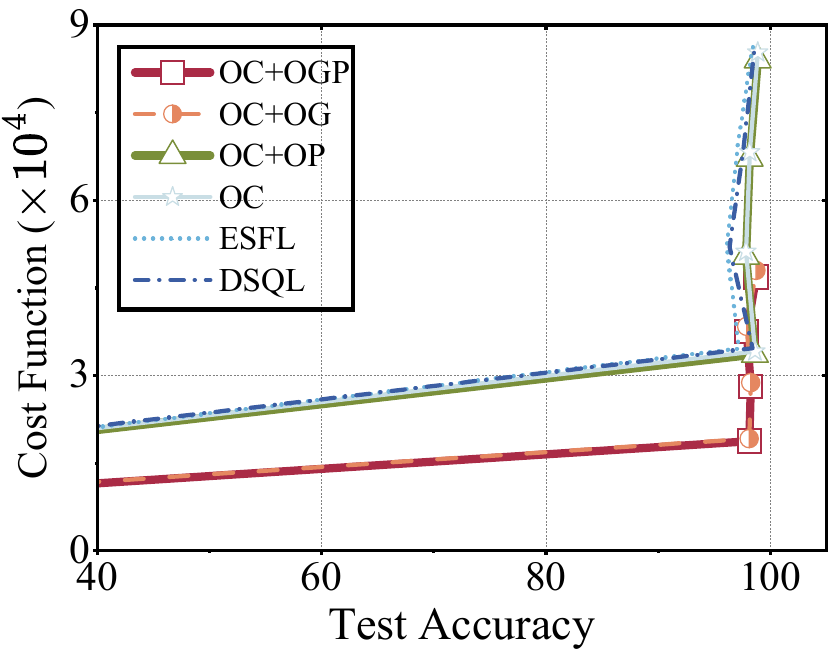}
	}
\subfigure[\centering Cost function with test accuracy (MNIST, non-IID, VGG-16).] {
\centering\includegraphics[width=4.1cm]{ 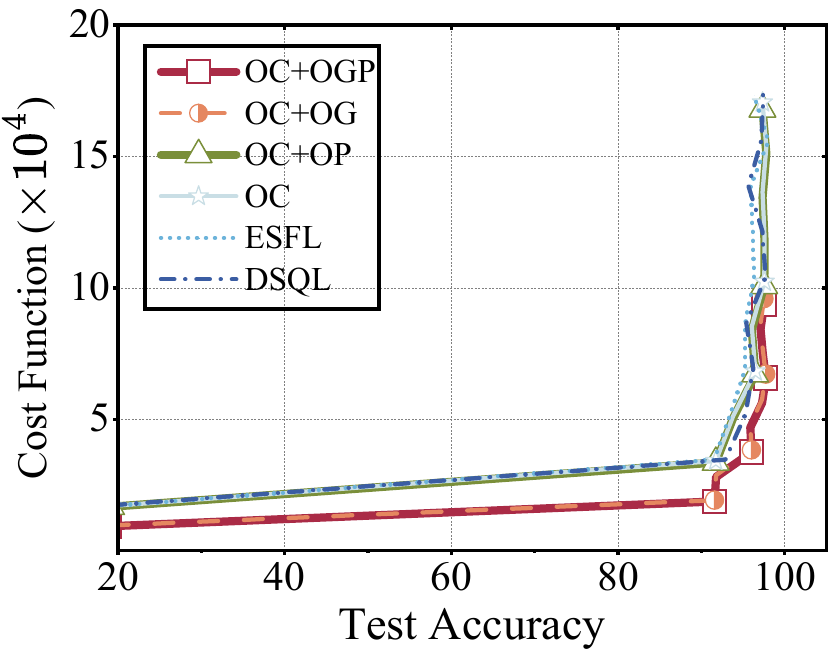}
	}
    \subfigure[\centering Cost function with test accuracy (CIFAR-10, IID, VGG-16).] {
\centering\includegraphics[width=4.1cm]{ 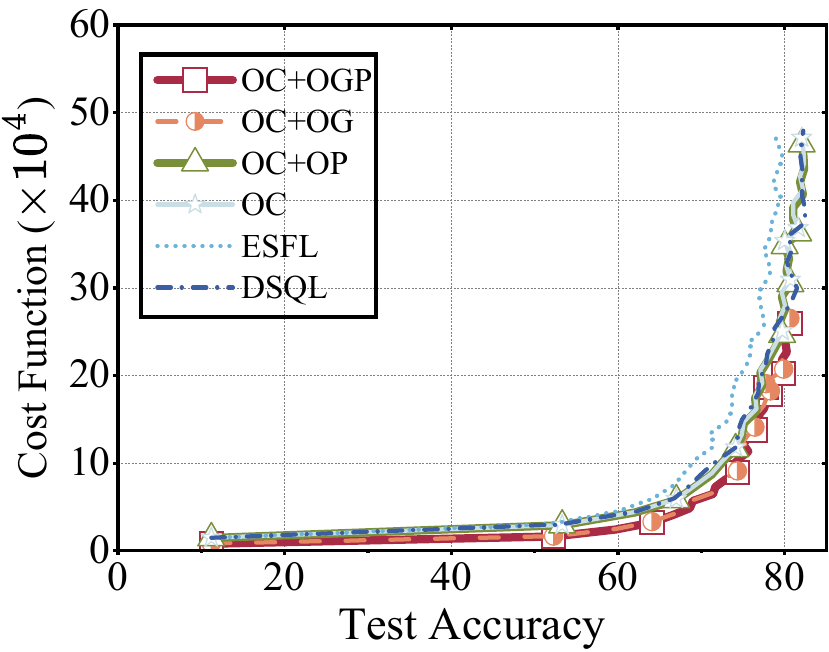}
	}
    \subfigure[\centering Cost function with test accuracy (CIFAR-10, non-IID, VGG-16).] {
\centering\includegraphics[width=4.1cm]{ 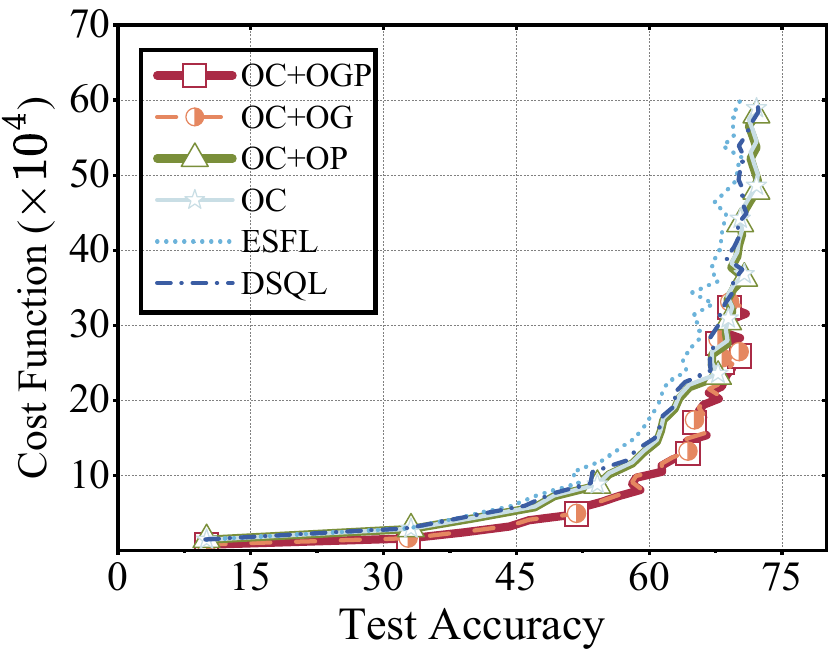}
	}
			\centering\caption{ Evaluating training performance on CIFAR-10 and MNIST using ResNet-50 and VGG-16 ($\lambda=100$).} \label{Training_curves}
\end{figure*}

\section{Simulations}\label{simulations}
\subsection{Simulation Setup}
We use the datasets MNIST \cite{6296535} and CIFAR-10~\cite{lecun1998mnist} to evaluate the performance of the proposed method  for the models ResNet-50~\cite{he2015deepresiduallearningimage} and VGG-16~\cite{simonyan2015deepconvolutionalnetworkslargescale} under both IID and non-IID data distributions.  
Notably, for the non-IID setting, each client is assigned a primary class from which approximately 70\% of its samples are drawn, while the remaining samples are sourced from other classes.
We consider a network of \(N=20\) clients, unless specified otherwise. Each client uses a mini-batch size of \(32\) with a learning rate \(\gamma=0.001\). 
 The algorithms are run on a workstation equipped with an AMD Ryzen Threadripper PRO 5975WX and NVIDIA GeForce RTX 4090.

 In the experiments, the tradeoff coefficient is set to $\lambda=100$ to balance training latency and energy consumption. The client-side computation capability is heterogeneous, with $n_i$ uniformly selected from $[0.2,3]\times10^{4}$ FLOPs/cycle, while the edge server is configured with $n_{\mathrm S}=3\times10^{4}$ FLOPs/cycle. The computation-energy coefficients are set to $G_i=1\times10^{-26}\,\mathrm{W/(cycle/s)^3}$ for clients and $G_{\mathrm S}=2\times10^{-26}\,\mathrm{W/(cycle/s)^3}$ for the server. 
For the NVIDIA Jetson Orin Nano client platform, the computing frequency is configured within the measured graphics-clock range, namely from $f_i^{\min}=50$ MHz to $f_i^{\max}=625$ MHz. For the NVIDIA GeForce RTX 4090 server platform, the computing frequency is configured within the supported graphics-clock range measured on our testbed, namely from $f_{\mathrm S}^{\min}=210$ MHz to $f_{\mathrm S}^{\max}=3.135$ GHz~\cite{jetson_orin_nano_spec,nvidia_smi_doc}.
 For wireless communication, the uplink and downlink bandwidths are set to $B_i^{\mathrm U}=100$ MHz and $B_i^{\mathrm D}=200$ MHz, respectively. The transmit power of each client and the edge server are bounded within $[20,33]$ dBm. 
The uplink channels are modeled as independent Rayleigh fading channels, where the average channel power gain of client $i$ is denoted by $|h_i|^2$ and generated within $[0.1,1]\times10^{-6}$. Similarly, the average downlink channel power gain is denoted by $|h_{\mathrm S}|^2$ and fixed at $1\times10^{-6}$~\cite{9461628}. The noise power spectral density is set to $N_0=3.98\times10^{-21}\,\mathrm{W/Hz}$. 
 For the reader's convenience, the detailed simulation parameters are summarized in Table \ref{tab:simulation_parameters}.

To demonstrate the advantages of our proposed method ``Optimal Cut and Optimal GPU Frequency Scaling and Power Control'' (\textbf{OC + OGP}), we compare it with five benchmarks: 1) ``Optimal Cut and Optimal GPU Frequency Scaling'' (\textbf{OC + OG}), which optimizes the model partition point and GPU frequency; 2) ``Optimal Cut and Optimal Power Control'' (\textbf{OC + OP}), which jointly optimizes the model partition point and transmit power; 3) 
``Optimal Cut'' (\textbf{OC}), which optimizes only the model partition;  4) efficient split federated learning (\textbf{ESFL})~\cite{10522472}, which aims to reduce training latency by optimizing client-side workload and server-side workload;
and 5) dynamic-step Q-Learning method (\textbf{DSQL})~\cite{10526451}, which employs a reinforcement learning method to reduce the weighted sum of training time and energy consumption. 
In our experiments, the total cost is defined as the cumulative cost incurred until the model reaches 70\% test accuracy on CIFAR-10 under both IID and non-IID data distributions.

 \begin{table}[t]
\small
\centering
\caption{Simulation Parameters}
\label{tab:simulation_parameters}
\begin{tabular}{|c|c|c|c|}
\hline
\textbf{Parameter} & \textbf{Value} & \textbf{Parameter} & \textbf{Value} \\
\hline
$C$ & $20$ &$N_0$   &  \makecell{$3.98\times 10^{-21}$ \\$\mathrm{W/Hz}$}  \\
\hline
$J$ & $5$ &  $\lambda$ & $100$ \\
\hline
$n_i$ & \makecell{$[0.2,3]\times 10^{4}$\\ $\mathrm{FLOPs/cycle}$}  & $n_{\mathrm S}$ & \makecell{$3\times 10^{4}$\\$\mathrm{FLOPs/cycle}$} \\
\hline
$G_i$ &
\makecell{$1\times10^{-26}$\\$\mathrm{W/(cycle/s)^3}$} &
$G_{\mathrm S}$ &
\makecell{$2\times10^{-26}$\\$\mathrm{W/(cycle/s)^3}$} \\
\hline
$f_{i}^{\max}$ & $625$ $\mathrm{MHz}$ & $f_{i}^{\min}$ & $50$ $\mathrm{MHz}$ \\
\hline
$f_{\mathrm S}^{\max}$ & $3.135$ $\mathrm{GHz}$ & $f_{\mathrm S}^{\min}$ & $210$ $\mathrm{MHz}$ \\
\hline
$B_i^{\mathrm{U}}$ & $100$ $\mathrm{MHz}$ & $B^{\mathrm{D}}_i$ & $200$ $\mathrm{MHz}$ \\
\hline
$P_i^{\max}$ & $33$ $\mathrm{dBm}$ & $P_i^{\min}$ & $20$ $\mathrm{dBm}$ \\
\hline
$P_{\mathrm S}^{\max}$ & $33$ $\mathrm{dBm}$ & $P_{\mathrm S}^{\min}$ & $20$ $\mathrm{dBm}$ \\
\hline
$|h_i|^2$ & $[0.1,1] \times 10^{-6}$ & $ |h_{\mathrm{S}}|^2$ & $1\times10^{-6}$ \\
\hline
\end{tabular}
\end{table}
 
\begin{figure}[t!]
        \subfigure[\centering Cost function vs. uplink bandwidth.] {
\centering\includegraphics[width=3.85cm]{ 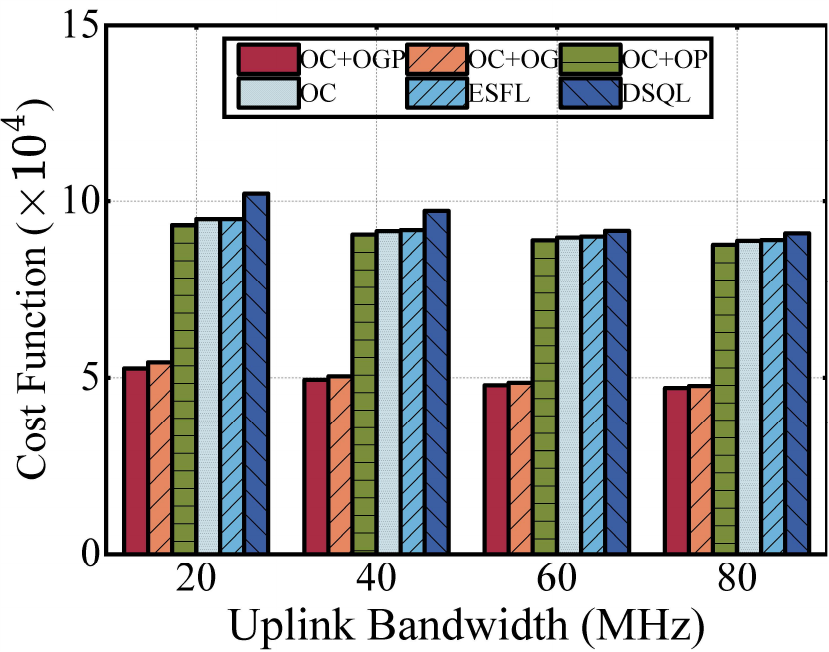}
	}
\subfigure[\centering Cost function vs. downlink bandwidth.] {
\centering\includegraphics[width=3.85cm]{ 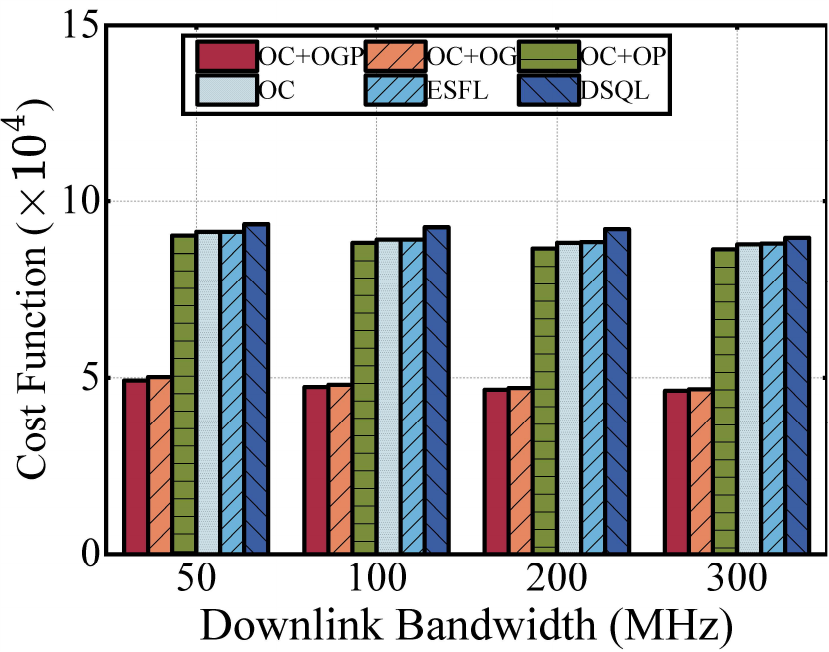}
	}
	\centering\caption{ The impact of uplink bandwidth and downlink bandwidth on cost function.} 
	\label{uplink_downlink}
\end{figure}

\begin{figure}[t!]
        \subfigure[\centering Cost function vs. client-side GPU frequency.] {
\centering\includegraphics[width=3.85cm]{ 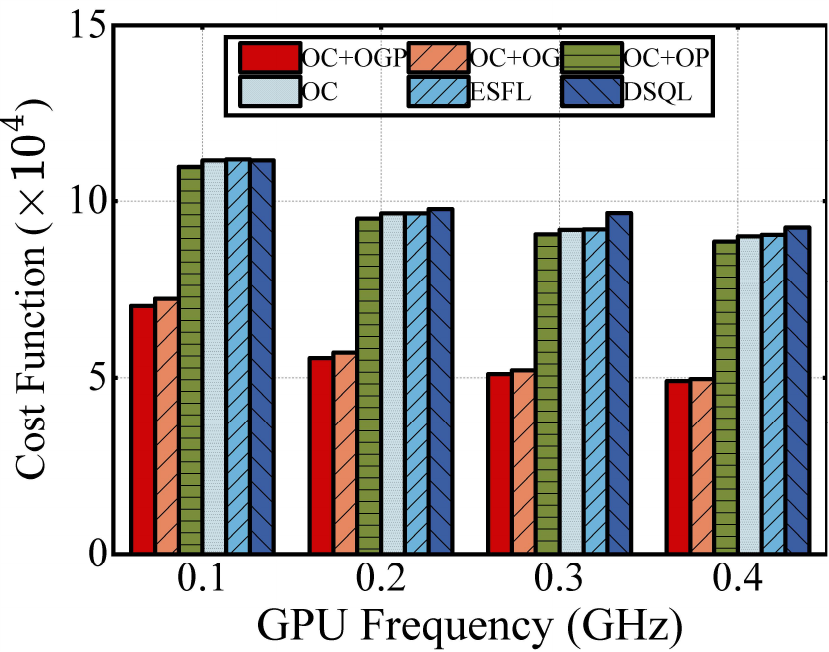}
	}
\subfigure[\centering Cost function vs. server-side GPU frequency.] {
\centering\includegraphics[width=3.85cm]{ 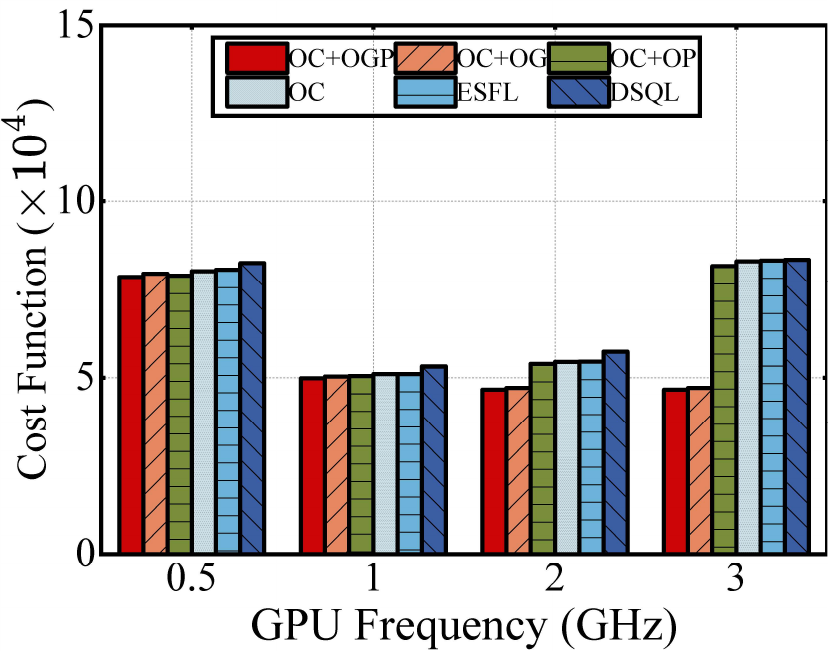}
	}
	\centering\caption{ The impact of GPU frequency on cost function.} 
		\label{compute}
\end{figure}

\begin{figure}[t!]
\subfigure[\centering Client-side computing uncertainty.] {
\centering\includegraphics[width=3.85cm]{ 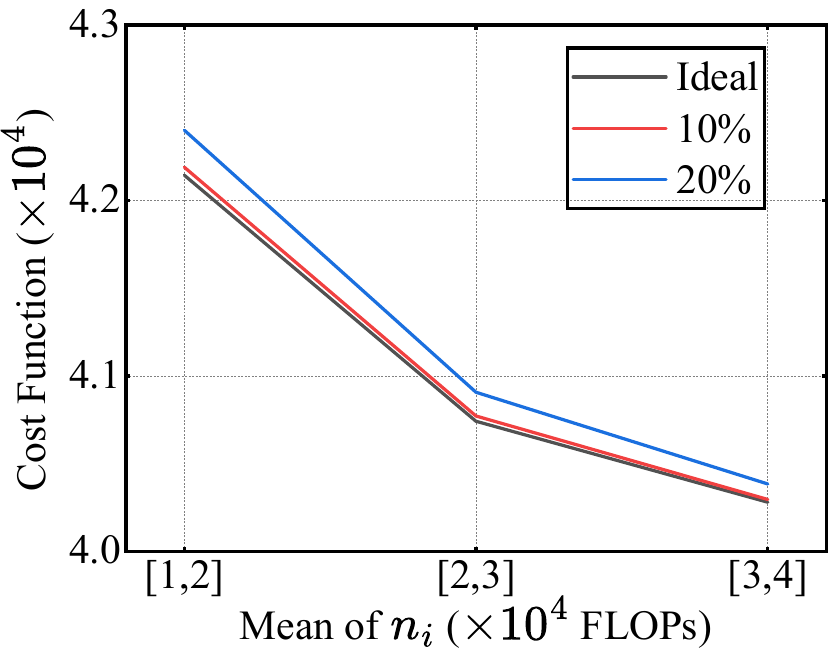}
	}
         \subfigure[\centering Server-side computing uncertainty.] {
\centering\includegraphics[width=3.85cm]{ 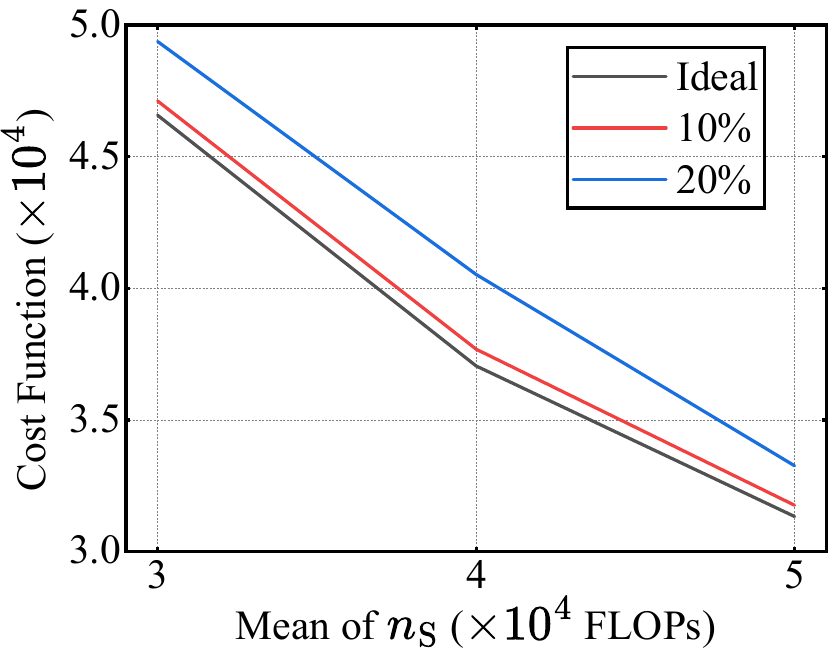}
	}
    \subfigure[\centering Uplink transmission uncertainty.] {
\centering\includegraphics[width=3.85cm]{ 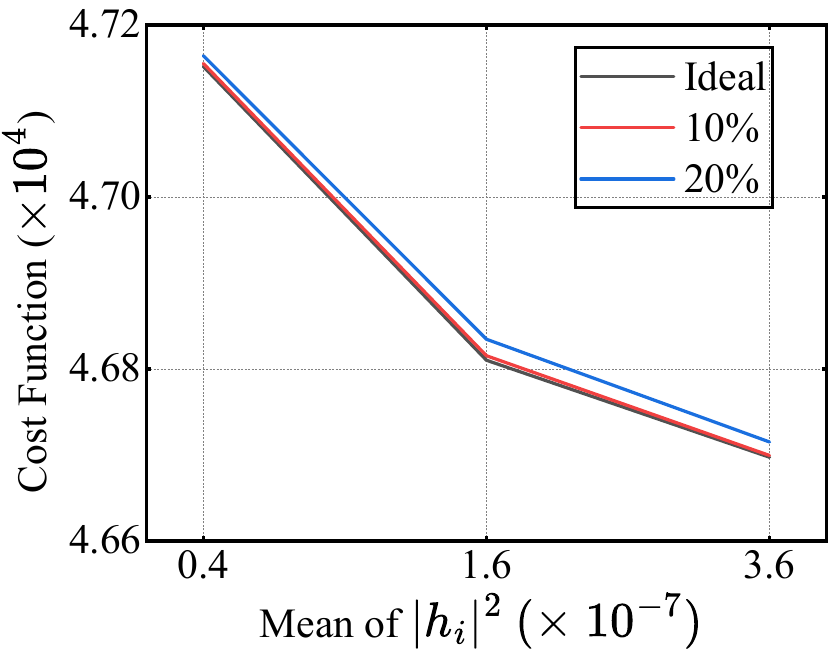}
	}
         \subfigure[\centering Downlink transmission uncertainty.] {
\centering\includegraphics[width=3.85cm]{ 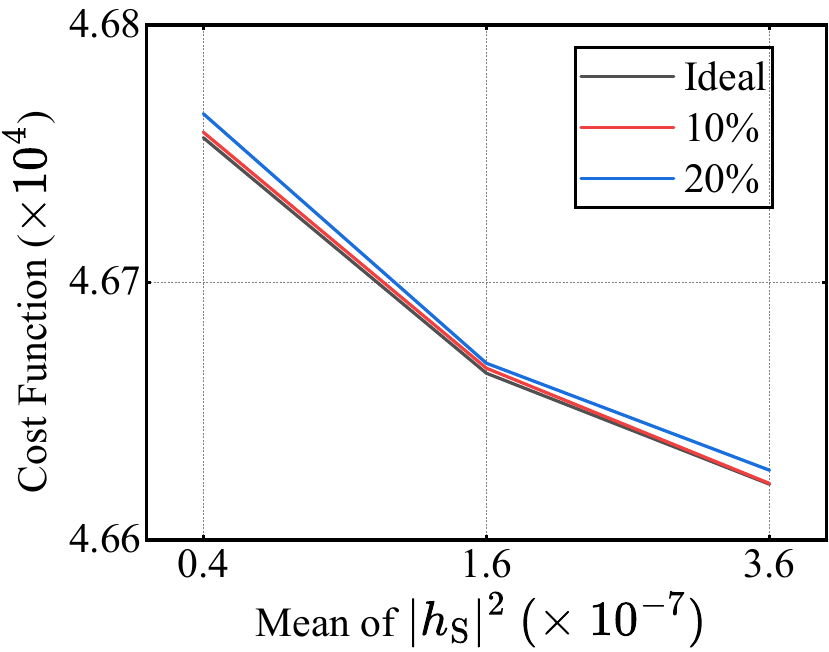}
	}
	\centering\caption{ The impact of network resource fluctuation on cost function.} 
		\label{flu}
\end{figure}

\begin{figure}[t!]
         \subfigure[\centering The impact of number of clients on cost function.] {
\centering\includegraphics[width=3.85cm]{ 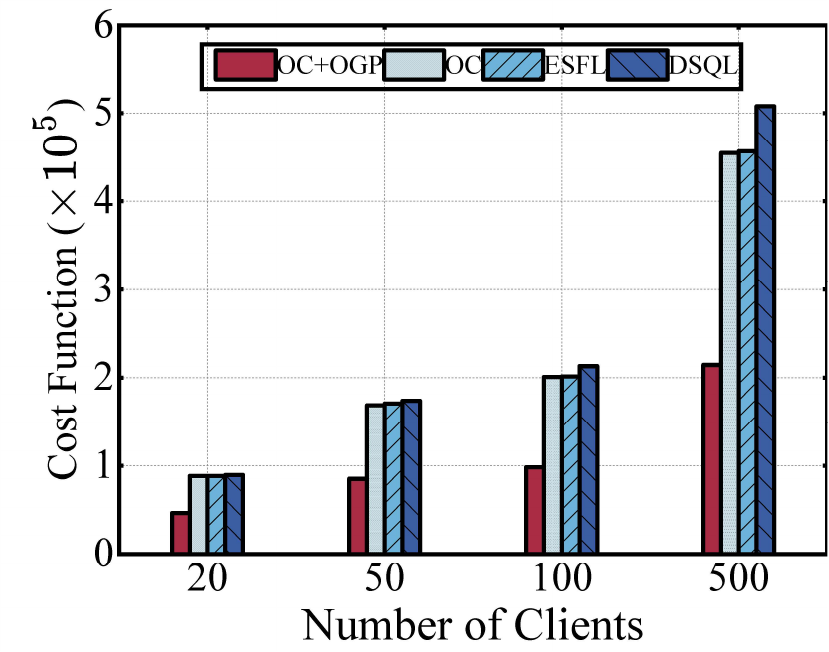}
	}
         \subfigure[\centering The impact of number of clients on running time.] {
\centering\includegraphics[width=3.85cm]{ 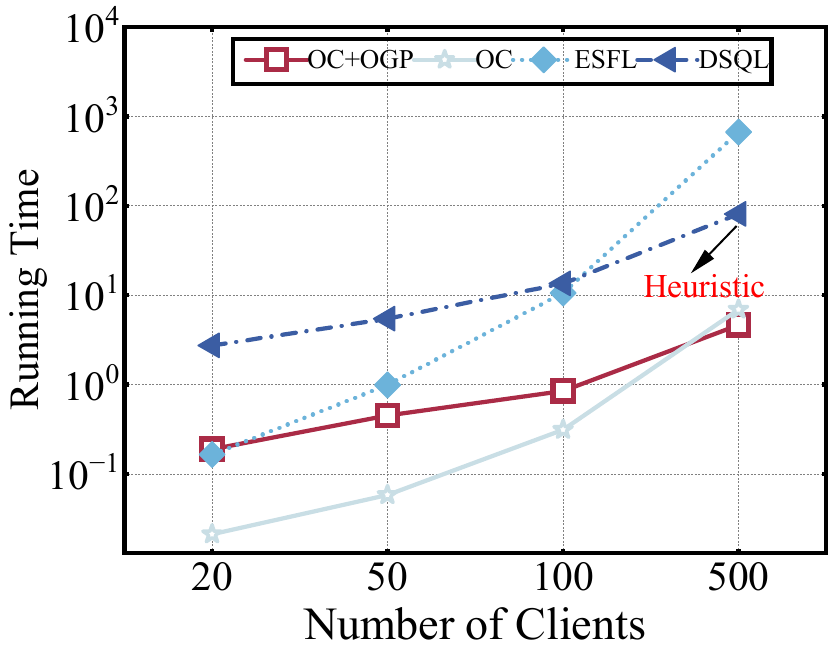}
	}
	\centering\caption{ The impact of number of clients on cost function and running time.} 
		\label{number_of_clients}
\end{figure}

 \subsection{Performance Comparison}
Fig.~\ref{Training_curves} evaluates the training performance by showing the cost function against test accuracy using ResNet-50 and VGG-16 models on both MNIST and CIFAR-10 datasets. The proposed \textbf{OC+OGP} method consistently achieves the lowest cost for attaining a given test accuracy, demonstrating that the proposed method remains effective across different datasets, data distributions, and model architectures. It is also observed that the cost generally increases with the test accuracy, and the increase becomes faster as the test accuracy approaches convergence.

Fig.~\ref{uplink_downlink} illustrates the impact of varying uplink and downlink bandwidths on the total cost function. As the uplink or downlink bandwidth increases, the overall cost gradually decreases for all compared baselines, since larger bandwidth improves transmission efficiency and shortens the communication time. Moreover, the proposed \textbf{OC+OGP} consistently achieves the lowest cost under all evaluated bandwidth settings, indicating its effectiveness under different communication resource conditions.
Fig.~\ref{compute} illustrates the impact of the maximum allowable GPU frequency on the cost function. In Fig.~\ref{compute}(a), increasing the maximum allowable client-side GPU frequency continuously reduces the cost for all compared methods, which indicates that, within the tested range, enlarging the feasible client-side frequency set is beneficial for cost reduction.  Fig.~\ref{compute}(b) shows that enlarging the maximum allowable server-side GPU frequency does not necessarily reduce the cost. For the proposed \textbf{OC+OGP} and \textbf{OC+OG}, the cost remains unchanged once the feasible range already includes the server-side GPU frequency that gives the lowest cost, since the proposed method chooses the cost-minimizing frequency instead of simply using the maximum available frequency. By contrast, the baseline schemes incur higher cost at large server-side frequency values because they operate at the maximum allowable GPU frequency.

In edge networks, computing resources and channel conditions may fluctuate during training, resulting in a mismatch between the measured parameters used for optimization and the actual operating conditions. To evaluate the impact of such uncertainty, we model the computing capabilities and average channel power gains as Gaussian random perturbations around their nominal values, with different coefficients of variation (CV)~\cite{10053757,yoo2024modeling}. Specifically, we consider uncertainty in four key parameters: $n_i$, $n_{\mathrm S}$, $|h_i|^2$, and $|h_{\mathrm S}|^2$, where $n_i$ and $n_{\mathrm S}$ denote the numbers of GPU FLOPs per cycle of client $i$ and the edge server, respectively, and $|h_i|^2$ and $|h_{\mathrm S}|^2$ denote the average uplink and downlink channel power gains, respectively. In our experiments, the CV is set to 10\% and 20\% to represent moderate and high uncertainty levels, respectively. For each CV level, we repeat the simulation over 200 random realizations. Fig.~\ref{flu} shows that the objective value increases as the CV grows, due to the mismatch between the optimized parameters and the actual conditions. Nevertheless, the proposed method consistently maintains competitive performance, demonstrating good robustness under practical resource uncertainty.

 \subsection{Running Time Comparison}
Fig.~\ref{number_of_clients} illustrates the impact of the number of clients on both the cost function and the running time. As shown in Fig.~\ref{number_of_clients}(a), the cost function generally increases as the number of clients grows for all methods compared.  The proposed method \textbf{OC+OGP} achieves a lower cost than baselines across all the numbers of clients tested. Moreover, the proposed \textbf{OC} method achieves the same cost as ESFL, while consistently outperforming DSQL.
Fig.~\ref{number_of_clients}(b) further shows that the proposed efficient \textbf{OC+OGP} and \textbf{OC}  methods require lower running time compared with ESFL and  DSQL. To sum up, the results verify the theoretical analysis that the proposed \textbf{OC} attains the optimal solution  with polynomial-time complexity. In addition, it shows that \textbf{OC+OGP} remains polynomial-time as the number of clients increases, which indicates that the proposed method is able to obtain high-quality approximate solutions efficiently.


\subsection{Impact of Hyperparameter $\lambda$}
\begin{figure}[t!]
        \subfigure[\centering Energy minimization ($\lambda =0$).] {
\centering\includegraphics[width=3.85cm]{ 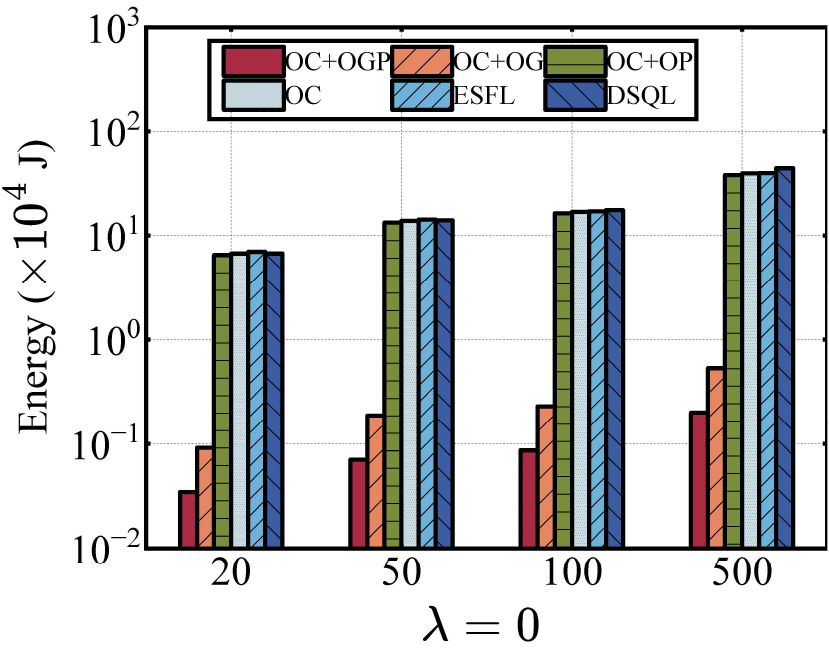}
	}
\subfigure[\centering Latency minimization ($\lambda =\infty$).] {
\centering\includegraphics[width=3.85cm]{ 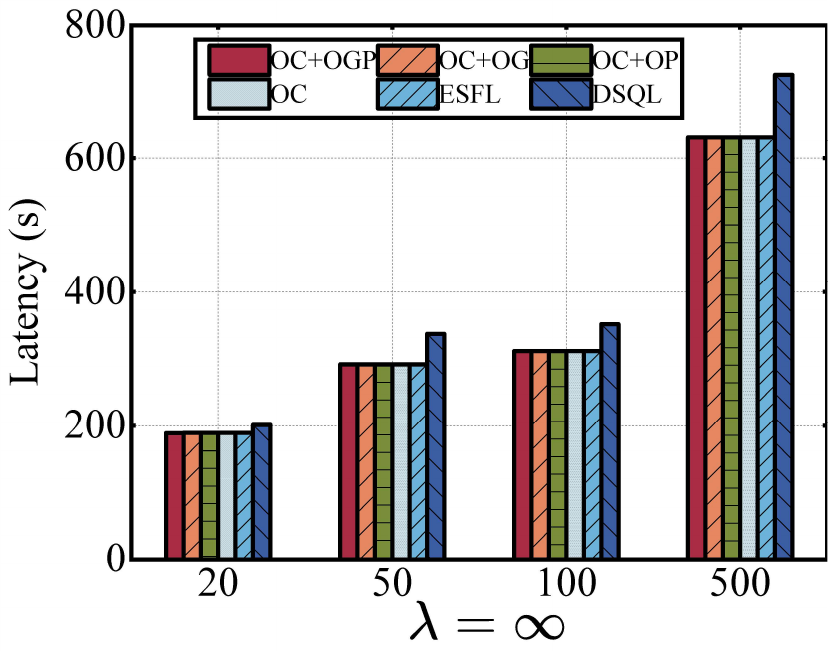}
	}
	\centering\caption{ Performance comparison under $\lambda =0$ and $\lambda =\infty$.} 
	\label{lambda-0-in}
\end{figure}
Fig.~\ref{lambda-0-in} further examines the two extreme cases. When 
$\lambda=0$, i.e., under energy-only optimization, the proposed \textbf{OC+OGP} achieves the lowest energy among all compared methods. When $\lambda=\infty$, i.e., under latency-only optimization, \textbf{OC+OGP} achieves nearly the same latency as \textbf{OC+OG}, \textbf{OC+OP}, \textbf{OC}, and ESFL, while remaining lower than DSQL. 
This is because the latency-only objective drives the optimized GPU frequencies and transmit powers to their maximum feasible values, so the additional resource-control degrees of freedom in \textbf{OC+OGP}, \textbf{OC+OG}, and \textbf{OC+OP} do not further reduce latency beyond the cut-layer optimization. These results demonstrate that the proposed method can adapt to different optimization preferences and maintain favorable performance under both energy-oriented and latency-oriented settings.


\section{Conclusion}\label{conclusion}
In this paper, we investigated the joint optimization of model splitting and resource allocation for SFL in resource-constrained edge systems. We formulated the resulting design as a hybrid optimization problem that captures the coupled effects of computation, communication, and model-splitting decisions. For the special model-splitting case, we developed an optimal polynomial-time solution. For the general joint optimization problem, we reformulated it as a two-dimensional master problem and developed a computationally efficient approximation method with explicit theoretical guarantees. Numerical results demonstrated the effectiveness of the proposed approach and its advantage over existing baselines. Overall, the proposed framework provides an efficient method for resource optimization in SFL under the energy--latency tradeoff.

\bibliographystyle{IEEEtran}
\bibliography{citationlist}

\end{CJK}
\end{document}